\documentclass{article} % For LaTeX2e
\PassOptionsToPackage{dvipsnames}{xcolor}
\usepackage{iclr2027_conference,times}
\usepackage{xcolor}

\usepackage[utf8]{inputenc} % allow utf-8 input
\usepackage[T1]{fontenc}    % use 8-bit T1 fonts
\usepackage{hyperref}       % hyperlinks
\usepackage{url}            % simple URL typesetting
\usepackage{booktabs}       % professional-quality tables
\usepackage{amsfonts}       % blackboard math symbols
\usepackage{nicefrac}       % compact symbols for 1/2, etc.
\usepackage{microtype}      % microtypography

\usepackage{amsmath,amsfonts,bm}

\def\eqref#1{equation~\ref{#1}}
\def\1{\bm{1}}

\DeclareMathAlphabet{\mathsfit}{\encodingdefault}{\sfdefault}{m}{sl}
\SetMathAlphabet{\mathsfit}{bold}{\encodingdefault}{\sfdefault}{bx}{n}

\usepackage{times}
\usepackage{latexsym}
\usepackage{graphicx}
\usepackage{amsmath}
\usepackage{amssymb}
\usepackage{mathtools}
\usepackage{amsthm}
\usepackage{array}
\usepackage{longtable}
\usepackage[most]{tcolorbox}
\usepackage{bibentry}
\usepackage{bm} 
\usepackage{tabularx}
\usepackage{multirow} 
\usepackage{algorithm}
\usepackage{algorithmic}
\usepackage{appendix}
\usepackage{wrapfig}   % 用于环绕文本的图片
\usepackage{bbm}
\usepackage{booktabs}
\usepackage{xcolor}
\usepackage{rotating}
\definecolor{gainred}{HTML}{C62828}
\usepackage{enumitem}

\makeatletter
\def\blankfootnote{\xdef\@thefnmark{}\@footnotetext}
\makeatother

\definecolor{royalblue(web)}{rgb}{0.25, 0.41, 0.88}
\definecolor{blue-violet}{rgb}{0.54, 0.17, 0.89}
\definecolor{brightmaroon}{rgb}{0.76, 0.13, 0.28}
\definecolor{darkmagenta}{rgb}{0.55, 0.0, 0.55}
\definecolor{bleudefrance}{rgb}{0.19, 0.55, 0.91}
\definecolor{palatinateblue}{rgb}{0.15, 0.23, 0.89}
\definecolor{whitesmoke}{rgb}{0.96, 0.96, 0.96}
\definecolor{thulianpink}{rgb}{0.87, 0.44, 0.63}
\definecolor{amber(sae/ece)}{rgb}{1.0, 0.49, 0.0}
\definecolor{darkblue}{rgb}{0.0, 0.0, 0.55}
\definecolor{alizarin}{rgb}{0.82, 0.1, 0.26}
\definecolor{asparagus}{rgb}{0.53, 0.66, 0.42}
\definecolor{darkspringgreen}{rgb}{0.09, 0.45, 0.27}
\definecolor{columbiablue}{rgb}{0.61, 0.87, 1.0}
\definecolor{wildblueyonder}{rgb}{0.64, 0.68, 0.82}
\definecolor{trolleygrey}{rgb}{0.5, 0.5, 0.5}
\definecolor{paleaqua}{rgb}{0.74, 0.83, 0.9}
\definecolor{bubblegum}{rgb}{0.99, 0.76, 0.8}
\definecolor{coralred}{rgb}{1.0, 0.25, 0.25}
\definecolor{green(ryb)}{rgb}{0.4, 0.69, 0.2}
\definecolor{flame}{rgb}{0.89, 0.35, 0.13}
\definecolor{bittersweet}{rgb}{1.0, 0.44, 0.37}
\definecolor{darksalmon}{rgb}{0.91, 0.59, 0.48}
\definecolor{emerald}{rgb}{0.31, 0.78, 0.47}
\definecolor{green(pigment)}{rgb}{0.0, 0.65, 0.31}
\definecolor{codegreen}{rgb}{0,0.6,0}
\definecolor{codegray}{rgb}{0.5,0.5,0.5}
\definecolor{codepurple}{rgb}{0.58,0,0.82}
\definecolor{backcolour}{rgb}{0.96,0.96,0.94}
\definecolor{bluegray}{rgb}{0.3, 0.38, 0.47}

\tcbuselibrary{breakable}
\usepackage{listings}
\lstdefinestyle{mystyle}{
  basicstyle=\scriptsize\ttfamily,
  frame=single, % 添加边框
  columns=fixed, % 设置列宽为固定值
}

\lstdefinestyle{newstyle}{
  basicstyle=\footnotesize\ttfamily\color{codegreen},
  backgroundcolor=\color{backcolour},
  frame=shadowbox, % 需要 listings 和 fancyvrb 支持
  rulecolor=\color{red},
  frameround=tttt, % 四个角都圆角（仅当使用 frame=box 时有效）
  keywordstyle=\color{magenta},
  commentstyle=\color{green},
  stringstyle=\color{red},
  showstringspaces=false,
  numbers=left,
  numberstyle=\tiny\color{gray},
  breaklines=true
}

\newtcolorbox{theorybox}{
    colback=blue!2!white,
    colframe=blue!45!black,
    boxrule=0.5pt,
    arc=1mm,
    left=1.5mm,
    right=1.5mm,
    top=1mm,
    bottom=1mm,
    breakable
}

\usepackage{color}

\newtheorem{theorem}{Theorem}

\newcommand{\ours}{{\fontfamily{qpl}\selectfont OptiCom}}

\title{\ours{}: A Unified Framework for State-Conditioned Composition in LLM-Driven Optimization}

\author{%
\textbf{Chenxing Wei}$^{\dagger\S\heartsuit\diamond}$,
\textbf{Sichen Liu}$^{\circ\heartsuit}$,
\textbf{Lizhao Liu}$^{\heartsuit}$,
\textbf{Ningyuan Sun}$^{\heartsuit}$,
\textbf{Chen Bingzhou}$^{\heartsuit}$\\
\textbf{Ying He}$^{\spadesuit}$,
\textbf{Bo Jiang}$^{\heartsuit}$,
\textbf{Fei Yu}$^{\ddagger}$,
\textbf{Yao Shu}$^{\wr}$\thanks{Corresponding author.\quad
$\diamond$ Work done during an internship at ByteDance.}\\[6pt]
$^{\dagger}$School of Computing and Data Science,
The University of Hong Kong\\
$^{\circ}$Huazhong University of Science and Technology\\
$^{\S}$Shenzhen Loop Area Institute\\
$^{\wr}$Hong Kong University of Science and Technology (Guangzhou)\\
$^{\spadesuit}$ Shenzhen University\\
$^{\ddagger}$School of Information Technology, Carleton University\\
$^{\heartsuit}$ByteDance\\[3pt]
\texttt{weichenxing@connect.hku.hk},
\texttt{yaoshu@hkust-gz.edu.cn}
}

\iclrfinalcopy % Uncomment for camera-ready version, but NOT for submission.

\begin{document}

\maketitle
\begin{abstract}
Large language models (LLMs) are increasingly deployed to solve complex scientific and practical problems via iterative optimization. However, dynamically coordinating diverse search mechanisms as candidate quality, failure modes, and resource budgets evolve remains a critical open challenge. Targeted empirical diagnostics reveal that mechanism effectiveness is highly state-dependent. Motivated by this, we analyze how individual decisions drive final outcomes, decomposing the expected terminal improvement under a shared budget into cumulative decision opportunities minus cumulative selection losses. Guided by this opportunity-loss theoretical foundation, we propose \ours{}, a unified framework that represents LLM-driven optimizers within a shared configuration space: $C=(A,Q,O,E,M,S)$, corresponding to artifact, query, operator, evaluation, memory, and strategy. Operating within this space, a fast LLM-based Optimization Controller dynamically composes immediate mechanisms through structured Action Packages, while a slower Strategy Adapter refines long-term selection preferences, operator weights, and templates based on accumulated trajectory feedback. Comprehensive evaluations across 32 benchmark groups demonstrate the superiority of framework: \ours{} achieves an average Max-score rank of 1.72 among 14 evaluated configurations, securing the top score in 23 groups. Ultimately, these results highlight the broad applicability and high extensibility of \ours{} as a general-purpose paradigm for robust LLM test-time scaling.
\end{abstract}

\section{Introduction}
\label{introduction}

Large language models (LLMs) are increasingly used to improve solutions to scientific and practical problems through iterative generation and evaluation~\cite{snell2025scaling, liu2026skydiscover}. Applications span mathematical discovery~\cite{FunSearch2023, tsoukalas2026advancingmathematicsresearchaidriven}, computational efficiency~\cite{ouyang2025kernelbench}, and image generation~\cite{zhang2025itercomp}. These settings impose concrete requirements: a mathematical construction must remain feasible while improving a target bound~\cite{FunSearch2023, wei2026testtime}; a GPU kernel must preserve correctness while reducing execution time~\cite{ouyang2025kernelbench}; and a generated image must satisfy semantic and visual requirements~\cite{zhang2025itercomp}. Challenging reasoning benchmarks likewise distinguish plausible outputs from solutions that satisfy the task~\cite{chollet2026arcagi2newchallengefrontier, foundation2026arcagi3newchallengefrontier}. Despite differences in representation and evaluation, these diverse problems require searching for better candidates under constraints and limited computational resources~\cite{evolutionarysurveyandroadmap}.

Existing approaches provide complementary mechanisms for this search~\cite{zhang2026systematicsurveylargelanguage, gao2026surveyselfevolvingagentswhat}. OPRO conditions candidate generation on evaluated solutions and their scores~\cite{yang2024large}, while TextGrad propagates natural-language feedback to update variables in computational graphs~\cite{yuksekgonul2025optimizing}. FunSearch and AlphaEvolve combine program generation with evolutionary search and automated evaluation~\cite{FunSearch2023, novikov2025alphaevolvecodingagentscientific}, alongside broader progress in algorithm discovery~\cite{AlphaDev2023, liu2024eoh, zheng2025monte, ye2024reevo}. Reflexion and ExpeL reuse experience to guide later attempts~\cite{shinn2023reflexion, zhao2024expel}, and GEPA combines reflection with complementary candidate reuse~\cite{agrawal2025gepa}. These methods illustrate different ways to organize proposals, feedback, and experience into an improvement loop~\cite{wei2026words}.

From an optimization perspective, this improvement loop involves a sequence of interconnected decisions: what information to acquire, which candidates to revise, how to modify them, how thoroughly to evaluate them, and what experience to retain. Crucially, the utility of these mechanisms shifts dynamically during optimization. An invalid candidate may require targeted repair; a feasible but stagnant solution might benefit from recombination or broader exploration; and a high-quality candidate often warrants a precise, structure-preserving local update. This perspective connects LLM-driven optimization to classical studies of problem-dependent search performance and hyper-heuristics~\cite{wolpert1997nofree, burke2013hyper}. It motivates treating the decisions that govern candidate improvement as objects of optimization themselves, raising a central question: \emph{How can we represent the decisions made by different optimizers in a common space, and adaptively compose them as the task, optimization state, and remaining budget evolve?}

Recent progress in shared abstractions and optimizer adaptation provides a foundation for this view~\cite{khattab2024dspy, cemri2026adaevolveadaptivellmdriven, liu2026evoxmetaevolutionautomateddiscovery, metatextgrad2025, zhang2025darwin}. Building on these directions, we seek a common representation of optimization responsibilities and principles for coordinating them. Our diagnostic study provides empirical evidence that mechanism preferences vary across the constructed optimization states (Section~\ref{sec:motivation}). To rigorously understand how these dynamic decisions affect the complete optimization process, we analyze general policies under a shared budget and terminal evaluation criterion (Section~\ref{sec:theory}). We decompose the expected terminal improvement into \emph{cumulative decision opportunities} ($\Delta_t$) minus \emph{cumulative selection losses} ($\ell_t$). Under a near-optimal selection assumption, we derive a sufficient condition for positive terminal improvement. This analysis accounts for exploration costs and unfavorable decisions through their cumulative effect, yielding two complementary design principles: make valuable alternatives available, and improve their selection using the current state and accumulated evidence.

Guided by this theoretical perspective, we introduce \ours{}, a unified framework that represents optimizers as configurations $C=(A,Q,O,E,M,S)$ and adapts their component compositions. An Optimization Controller dynamically selects Action Packages based on the current state, while a slower Strategy Adapter refines selection preferences over time. Across 32 benchmark groups, \ours{} obtains the lowest average rank (1.72 for Max, 2.08 for Mean) among 14 evaluated configurations, supported by additional checks using official SkyDiscover implementations on seven benchmarks. Our contributions are fourfold:
\begin{itemize}[topsep=0pt,leftmargin=8mm,itemsep=0pt]
    \item \textbf{Shared representation.} We define optimization decisions across six functional dimensions (AQOEMS), enabling disparate existing methods and complementary mechanisms to interact within a common, extensible space (Section~\ref{sec:configuration-space}).
    \item \textbf{Theoretical analysis.} We formalize terminal improvement through the lens of cumulative opportunities and selection losses, yielding theoretically grounded principles for adaptive optimization (Section~\ref{sec:theory}).
    \item \textbf{Adaptive framework.} We instantiate these principles via \ours{}, featuring state-conditioned Action Package composition, a task-specific execution harness, and two-timescale strategy adaptation (Section~\ref{sec:framework}).
    \item \textbf{Empirical validation.} Through cross-domain evaluations, state-conditioned motivation studies, targeted ablations, and token-aware case studies, we evaluate final quality and examine the contributions of adaptation, experience reuse, and continued search expenditure (Section~\ref{sec:experiments}).
\end{itemize}

\section{A Shared Configuration Space for LLM-Driven Optimizers}
\label{sec:configuration-space}

\subsection{Optimization Concepts and the AQOEMS Representation}
\label{sec:aqoems}

Formally, given a task $d$, an optimizer explores a candidate space $\mathcal{X}_d$ under a computational budget $B$ to find solutions for:
\begin{equation}
    \max_{x\in\mathcal{X}_d} f_d(x)
    \qquad \text{subject to } x\in\mathcal{F}_d,
    \label{eq:optimization-objective}
\end{equation}
where $\mathcal{F}_d$ and $f_d$ denote the feasible set and the objective function, respectively. We conceptualize the underlying mechanisms and coordinating rules of any such optimizer as a unified configuration tuple:
\begin{equation}
    C=(A,Q,O,E,M,S).
    \label{eq:configuration}
\end{equation}
Here, $A$ defines the artifact representation; $Q$ dictates information acquisition and context construction; $O$ governs candidate generation and updating operators; $E$ specifies quality and feasibility evaluation; $M$ manages retained experience and candidate archives; and $S$ provides the overarching strategy for search coordination and resource allocation.

Crucially, AQOEMS serves as a functional taxonomy rather than a strict architectural blueprint for six isolated modules or sequential steps. Components may share implementations or be entirely omitted (taking null forms). The instantiation of these functions dynamically adapts to the current optimization state---such as candidate quality, encountered failures, historical experience, and remaining budget. For instance, while $E$ assesses the candidate and generates feedback, $Q$ selectively routes this evidence for future inspection. Consequently, even a visually "fixed" configuration $C$ can encapsulate state-dependent or adaptive rules within its strategy $S$, meaning it need not execute the identical action at every iteration. Further details on component responsibilities and boundaries are provided in Appendix~\ref{app:mapping-criteria}.

\subsection{Existing Optimizers through the AQOEMS Lens}
\label{sec:existing-optimizers}

AQOEMS naturally captures candidate--feedback loops across diverse optimization families, encompassing stateful, stochastic, and internally adaptive mechanisms. We organize related work into three overlapping directions, with one implemented representative from each detailed in Table~\ref{tab:optimizer-mapping}.

\begin{table}[t]
\vspace{-7mm}
    \centering
    \small
    \setlength{\tabcolsep}{4pt}
    \renewcommand{\arraystretch}{1.12}
    \caption{Representative AQOEMS mappings across three research directions. Each method has a framework instantiation; mapping and alignment details appear in Appendix~\ref{app:optimizer-mapping}.}
    \label{tab:optimizer-mapping}
    \begin{tabularx}{\linewidth}{
        @{}l
        >{\raggedright\arraybackslash}X
        >{\raggedright\arraybackslash}X
        >{\raggedright\arraybackslash}X@{}
    }
        \toprule
        & \textbf{Feedback-guided revision and experience reuse} & \textbf{Evolutionary search and adaptive control} & \textbf{Shared abstractions and meta-optimization} \\
        \addlinespace
        & TextGrad~\cite{yuksekgonul2025optimizing} & AdaEvolve~\cite{cemri2026adaevolveadaptivellmdriven} & DSPy~\cite{khattab2024dspy} \\
        \midrule
        $A$ & Optimizable graph variables & Candidate programs & LM program instructions and demonstrations \\
        \addlinespace
        $Q$ & Graph context and propagated feedback & Selected programs and search evidence & Training examples and execution traces \\
        \addlinespace
        $O$ & Textual-feedback-guided updates & LLM-generated program variations & Instruction and demonstration updates \\
        \addlinespace
        $E$ & Objective feedback & Program fitness & User-specified task metric \\
        \addlinespace
        $M$ & Graph and update state & Populations and improvement statistics & Optimizer-dependent candidates and search state \\
        \addlinespace
        $S$ & Feedback propagation and update rules & Adaptive exploration, allocation, and guidance & Selected optimizer's compilation and search rules \\
        \bottomrule
    \end{tabularx}
    \vspace{-3mm}
\end{table}

\textbf{Feedback-guided revision and experience reuse.}
TextGrad connects evaluation ($E$) to variable updates ($O$) through graph-specific context ($Q$)~\cite{yuksekgonul2025optimizing}. OPRO uses scored candidate history~\cite{yang2024large}, Reflexion retains verbal reflections~\cite{shinn2023reflexion}, and ExpeL extracts reusable knowledge~\cite{zhao2024expel}. These methods exemplify distinct strategies for acquiring and applying evidence via $Q$, $M$, and $O$.

\textbf{Evolutionary search and adaptive control.}
AdaEvolve uses improvement statistics ($M$) to adapt exploration and allocation ($S$)~\cite{cemri2026adaevolveadaptivellmdriven}. FunSearch and AlphaEvolve combine program generation, automated evaluation, and population-based search~\cite{FunSearch2023,novikov2025alphaevolvecodingagentscientific}; GEPA adds reflective updates and complementary candidate reuse~\cite{agrawal2025gepa}; and EvoX evolves search procedures~\cite{liu2026evoxmetaevolutionautomateddiscovery}. These frameworks showcase rich interactions among $O$, $E$, $M$, and an internally adaptive $S$.

\textbf{Shared abstractions and meta-optimization.}
DSPy optimizes LM programs against user-specified metrics, with concrete operations and search rules supplied by the selected optimizer~\cite{khattab2024dspy}. Trace uses execution traces and feedback for generative optimization~\cite{trace2024}, while metaTextGrad improves optimizer prompts and compositions~\cite{metatextgrad2025}. 

Unlike existing universal frameworks that enforce static search pipelines or monolithic optimization APIs, \ours{} dynamically reshapes its underlying mechanism composition during execution to adapt to the evolving optimization state, while offering superior extensibility to integrate novel mechanisms. Methods enter the framework through these extensible candidate, feedback, memory, and action interfaces. We defer detailed discussions of existing larger framework compositions (e.g., LLM4AD~\cite{liu2024llm4ad}, SkyDiscover~\cite{liu2026skydiscover}, and optimize\_anything~\cite{agrawal2026optimizeanything}) to Appendix~\ref{app:expanded-related-frameworks}, and details regarding the scope of our contribution to Appendix~\ref{app:scope-of-contribution}.

\section{Motivating State-Conditioned Optimization}
\label{sec:motivation}

\begin{figure}[t]
\vspace{-5mm}
    \centering
    \includegraphics[width=\linewidth]{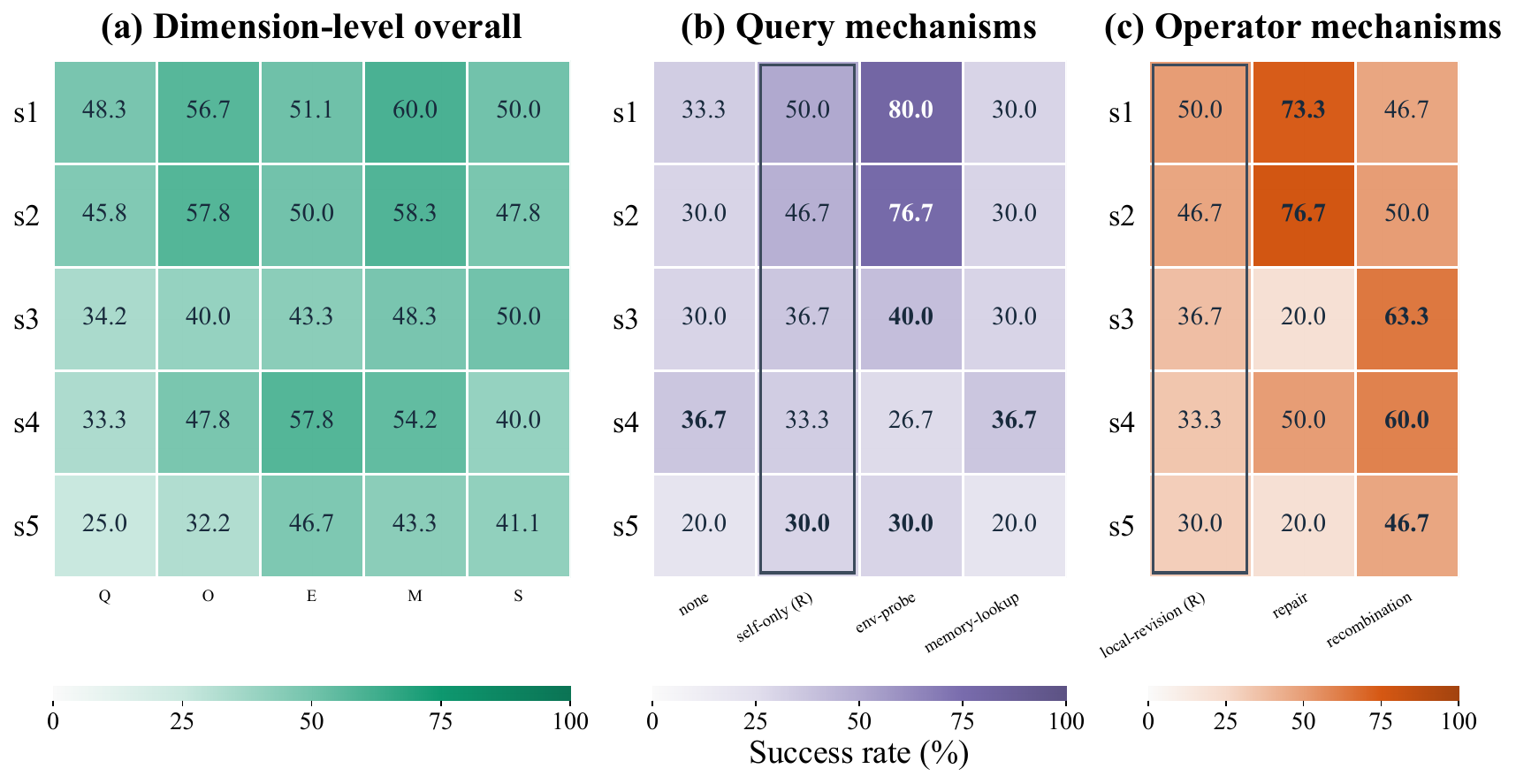}
    \caption{State-conditioned diagnostics on Heilbronn Triangle: three distinct situations per state and ten trials per mechanism--situation pair (30 trials per mechanism--state rate). \textbf{(a)} Averages over Q/O/E/M/S alternatives; \textbf{(b)} query and \textbf{(c)} operator comparisons. R marks the reference setting. Memory comparisons enable retrieval; other dimensions change one mechanism from the reference. Appendix~\ref{app:motivation-protocol} details the protocol and complete six-panel results.}
    \label{fig:motivation}
\vspace{-5mm}
\end{figure}

We examine state-dependent mechanism preferences on Heilbronn Triangle using 15 starting situations across five state families: s1 (execution/interface failures), s2 (failed task-specific checks), s3 (objective stagnation), s4 (limited remaining budget), and s5 (high performance variability). The strict intervention matrix changes one mechanism from the reference configuration ($Q2$ self-only + $O1$ local revision + $E2$ textual feedback + $M1$ no persistent memory + $S1$ fixed strategy). Memory comparisons additionally enable Q4 retrieval so that stored experience can affect later decisions; these are conditional comparisons rather than single-coordinate interventions. Each mechanism--situation pair contributes ten trials, with success defined by resolution of the starting condition (Appendix~\ref{app:motivation-protocol}). Figure~\ref{fig:motivation} shows changes in relative mechanism effectiveness. Environment probing exceeds self-only querying in s1 and s2 (80.0\% and 76.7\% versus 50.0\% and 46.7\%), but has a lower observed rate in s4 (26.7\% versus 33.3\%). Repair reaches 73.3\% and 76.7\% in s1 and s2, yet falls to 20.0\% under stagnation, where recombination reaches 63.3\%. Across the s2 situations, joint Q+O modification reaches 83.3\%, compared with 76.7\% for either single change and 46.7\% for the reference. Conditional mean iterations among successful trials are 2.1 for Q+O versus 2.9, 3.6, and 5.2 for Q-only, O-only, and the reference, respectively (Appendix~\ref{app:motivation-joint}). These descriptive results motivate state-conditioned selection and complementary composition; they do not establish universal preferences or a general interaction effect.

\section{Theoretical Analysis of Adaptive Optimization}
\label{sec:theory}

The motivation study illustrates how different decisions may be useful in different states. We now ask how those decisions affect the final outcome of an optimization run. Our analysis applies to a general optimization policy $\pi$, not specifically to \ours{}. It first identifies the common structure of terminal improvement and then bounds the effect of imperfect selection. Let $J_B(\pi)$ be expected terminal quality under budget $B$, with terminal utility in $[0,1]$. At history $h_t$, let $V^{\pi_0}(h_t)$ be the expected terminal quality obtained by continuing with a reference policy $\pi_0$, and let $q(h_t,a)$ be the value of executing $a$ before returning to $\pi_0$. Define
\begin{equation}
\begin{aligned}
q^\star(h_t)&=\max_{a\in\mathcal{A}(h_t)}q(h_t,a),\\
\Delta_t&=q^\star(h_t)-V^{\pi_0}(h_t),\\
\ell_t&=q^\star(h_t)-\mathbb{E}_{\pi}[q(h_t,A_t)\mid h_t].
\end{aligned}
\label{eq:opportunity-selection-loss}
\end{equation}
Here, $\Delta_t$ is the best available one-decision improvement relative to reference continuation, while $\ell_t$ is the expected loss in selecting among those actions. The complete process and reference continuation are specified in Appendix~\ref{app:theory-setup}.

\subsection{Cumulative Opportunity--Loss Decomposition}
\label{sec:cumulative-advantage}

Our first result specializes the performance-difference identity~\cite{schulman2015trust} to this budgeted process. It characterizes the complete run without assuming that individual decisions improve the incumbent. The proof is given in Appendix~\ref{app:global-improvement-proof}.

\begin{theorybox}
\begin{theorem}[Cumulative Opportunity--Loss Decomposition]
\label{thm:global-adaptive-value}
Under the shared budgeted process,
\begin{equation}
J_B(\pi)-J_B(\pi_0)
=
\mathbb{E}_{\pi}
\left[\sum_{t<T}(\Delta_t-\ell_t)\right].
\label{eq:global-value-decomposition}
\end{equation}
\end{theorem}
\end{theorybox}

\textbf{Remark.} A larger opportunity $\Delta_t$ benefits performance only when selection realizes the additional value. At a fixed history, expanding $\mathcal{A}(h_t)$ while preserving existing action values and costs cannot decrease $q^\star(h_t)$. However, if the selected-action distribution remains unchanged, any increase in $q^\star(h_t)$ raises $\Delta_t$ and $\ell_t$ equally, leaving $\Delta_t-\ell_t$ unchanged. Thus, catalog expansion should be paired with selection that exploits the added choices.

\subsection{Selection Quality and Global Improvement}
\label{sec:global-improvement-condition}

Following the use of explicit model-capability assumptions in LLM selection theory~\cite{chen2025provable}, suppose that, at each reached history,
\begin{equation}
\Pr_{\pi}\!\left(q^\star(h_t)-q(h_t,A_t)\leq\epsilon_t\mid h_t\right)
\geq1-\delta_t,
\qquad 0\leq\epsilon_t,\delta_t\leq1.
\label{eq:near-optimal-selection}
\end{equation}
Since terminal utility lies in $[0,1]$, the action-value gap is at most one, giving $\ell_t\leq(1-\delta_t)\epsilon_t+\delta_t$. Combining this bound with Theorem~\ref{thm:global-adaptive-value} gives a lower bound on terminal improvement, proved in Appendix~\ref{app:selection-guarantee-proof}. Detailed theoretical remarks and discussions on the applicability of these assumptions to our empirical implementations are provided in Appendix~\ref{app:theory-remarks}.

\begin{theorybox}
\begin{theorem}[Global Performance Bound]
\label{thm:global-improvement}
Under the conditional selection assumption and bounded terminal utility,
\begin{equation}
J_B(\pi)-J_B(\pi_0)
\geq
\mathbb{E}_{\pi}
\sum_{t<T}
\left[\Delta_t-(1-\delta_t)\epsilon_t-\delta_t\right].
\label{eq:global-improvement-bound}
\end{equation}
\end{theorem}
\end{theorybox}

\textbf{Remark.} For fixed $\delta_t$, reducing $\epsilon_t$ leaves a selection-loss bound approaching $\delta_t$. Thus, refining near-optimal choices alone may be insufficient when the probability of missing them remains high. This motivates addressing recurring poor selections alongside improving near-optimal choices; the appropriate allocation of effort depends on their attainable improvements and costs.

\textbf{Design implications.}
Together, the results motivate two complementary design principles: expose valuable alternatives to create opportunities ($\Delta_t$), and use diagnostic feedback alongside accumulated experience to improve selection ($\ell_t$). Allocating decision effort according to the current state aims to realize these opportunities while accounting for selection errors and resource costs. Section~\ref{sec:framework} instantiates these theory-motivated principles in \ours{}.

\section{The OptiCom Framework}
\label{sec:framework}

\begin{figure}[ht]
\vspace{-5mm}
    \centering
    \includegraphics[width=\textwidth]{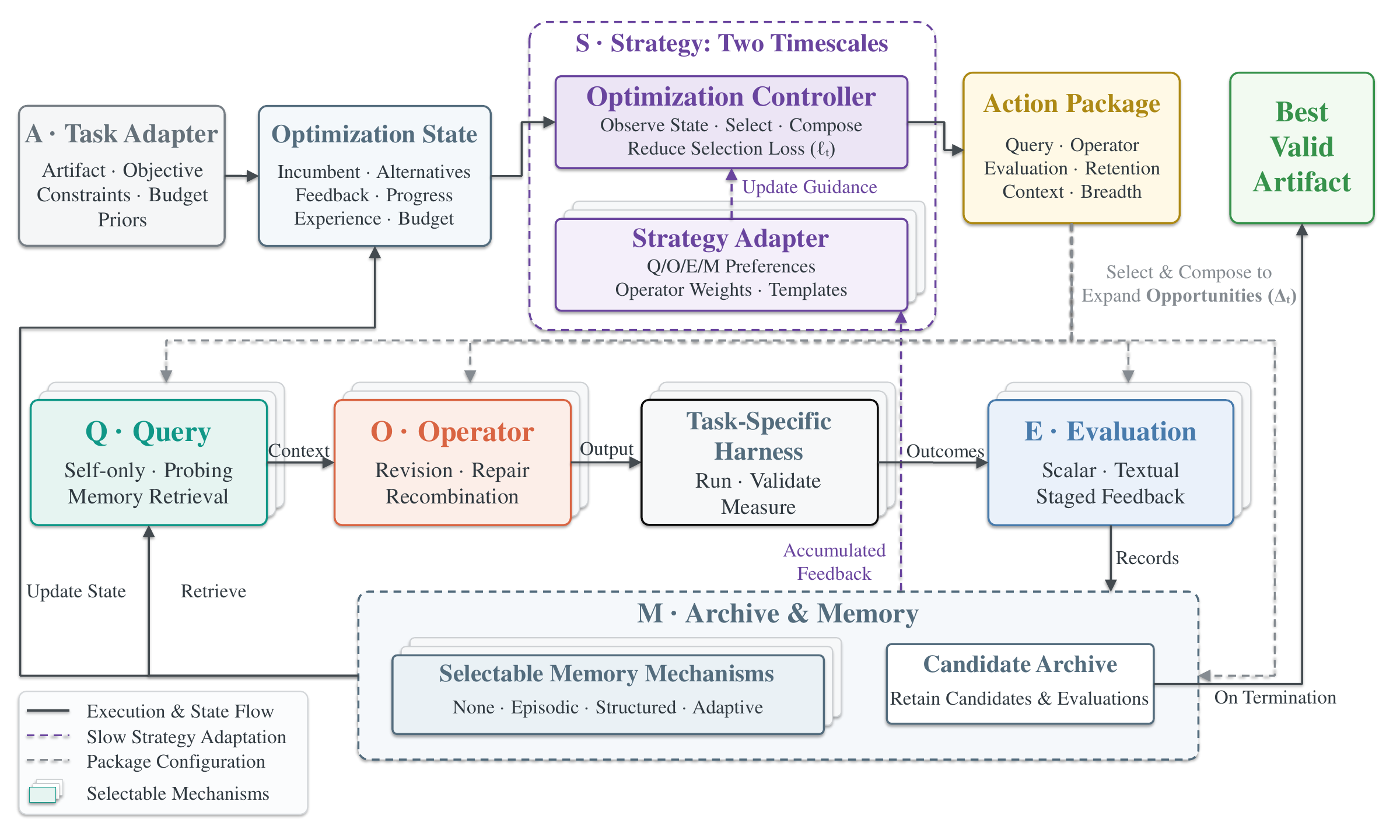}
    \caption{
        Overview of \ours{}.
        The Optimization Controller selects and composes Q/O/E/M mechanisms through an Action Package, while the slower Strategy Adapter updates the guidance used by subsequent Controller decisions.
        Solid arrows show execution and state flow; gray dashed arrows indicate configuration, and purple dashed arrows indicate slower strategy adaptation. Annotations involving $\Delta_t$ and $\ell_t$ highlight how specific architectural components are explicitly designed to expand search opportunities ($\Delta_t$) and reduce selection losses ($\ell_t$).
    }
    \label{fig:framework}
\vspace{-5mm}
\end{figure}

\textbf{Composing optimization mechanisms.}
The analysis in Section~\ref{sec:theory} identifies two factors governing adaptive optimization: the opportunities provided by available actions, $\Delta_t$, and the loss incurred when selecting among them, $\ell_t$. Guided by this decomposition, \ours{} dynamically combines complementary mechanisms within the shared configuration space. To maximize opportunities while mitigating selection loss, the framework operates through a continuous, state-conditioned execution loop driven by two core components: an Optimization Controller for immediate mechanism composition, and a slower Strategy Adapter for long-term guidance revision. Both components operate through LLM API calls with fixed model parameters (see Algorithm~\ref{alg:optimize-anything} in Appendix~\ref{app:algorithm}).

\textbf{State-conditioned execution and evaluation.}
At each iteration, the Optimization Controller observes the current optimization state---the incumbent, available alternatives, recent feedback, progress history, and remaining budget. It produces a structured Action Package specifying query and context scope, the update operator, generation breadth, evaluation effort, and retention policy. The framework acquires the requested evidence, generates candidates, and submits them to a task-specific Execution Harness. The harness runs the artifacts and returns the checks and measurements provided by the adopted benchmark pipeline (see Appendix~\ref{app:benchmark-details} for detailed benchmark configurations and pipeline specifications).

\textbf{Adapting strategy from accumulated experience.}
Following evaluation, candidates and outcomes are retained according to the selected archive and memory policies. The Strategy Adapter reviews accumulated successes, failures, and progress on a slower timescale, revising mechanism-selection preferences and prompt templates for subsequent Controller decisions. For example, recurring runtime errors may motivate greater emphasis on diagnostic queries. These updates change optimization guidance without modifying model weights or resetting the archive. The design uses accumulated evidence to inform later decisions; it does not estimate the theoretical continuation values or guarantee that every exploratory step will yield a long-term benefit. Complete component definitions, Action Package fields, and harness details appear in Appendix~\ref{app:implementation}.\author{
  Chenxing Wei$^{\dagger \S}$, Hong Wang$^{\circ}$, Ying He$^{ \dagger}$,  Fei Yu$^{ \ddagger}$, Yao Shu$^\wr$\thanks{corresponding author.}\\
$^\dagger$School of Computing and Data Science, The University of Hong Kong \\
$^\circ$University of Science and Technology of China \\
$^{\S}$Shenzhen Loop Area Institute \\
$^{\wr}$Hong Kong University of Science and Technology (Guangzhou) \\
$^{\ddagger}$School of Information Technology, Carleton University \\
\texttt{weichenxing@connect.hku.hk}, \texttt{yaoshu@hkust-gz.edu.cn}\\
}

\section{Experiments}
\label{sec:experiments}

\subsection{Experimental Setup}
\label{sec:experimental-setup}

We evaluate \ours{} across 32 benchmark groups spanning mathematics, systems, GPU kernels, algorithms, reasoning, creative generation, prompts, and quantum circuits, integrated from SkyDiscover~\cite{liu2026skydiscover}. The overall comparison comprises \ours{} and 13 method-inspired profiles (e.g., LATS, TextGrad, ProTeGi) executed within our shared framework to ensure evaluation consistency. Table~\ref{tab:main-results} additionally incorporates official SkyDiscover implementations of EvoX and AdaEvolve as external implementation comparisons. All main comparisons use Doubao-Seed-2.0-pro as the optimization backbone, capped at 30 iterations across five independent runs. We report both Mean and Max final scores. A common iteration cap does not equalize candidate counts, API calls, tokens, or elapsed time. Comprehensive details on scoring, aggregation rules, and model budgets are provided in Appendix~\ref{app:experimental-protocols}.

\begin{table*}[t]
    \centering
    \caption{
        Main results across seven representative benchmarks over five runs. Unqualified method names denote method-inspired framework profiles; ``official'' denotes the official SkyDiscover implementation. Scores follow task-specific scales. Higher scores are better. Bold and underlined values indicate the best and second-best distinct results.
    }
    \label{tab:main-results}
    \begingroup
    \small
    \setlength{\tabcolsep}{3pt}
    \renewcommand{\arraystretch}{1.15}
    \resizebox{\textwidth}{!}{%
    \begin{tabular}{@{}l*{14}{c}@{}}
        \toprule
        \textbf{Domain}
        & \multicolumn{2}{c}{Math}
        & \multicolumn{2}{c}{Systems}
        & \multicolumn{2}{c}{GPU}
        & \multicolumn{2}{c}{Algorithms}
        & \multicolumn{2}{c}{Reasoning}
        & \multicolumn{2}{c}{Prompts}
        & \multicolumn{2}{c}{Quantum} \\
        \midrule
        \textbf{Benchmark}
        & \multicolumn{2}{c}{Heilbronn}
        & \multicolumn{2}{c}{LLM-SQL}
        & \multicolumn{2}{c}{TriMul}
        & \multicolumn{2}{c}{Frontier-CS}
        & \multicolumn{2}{c}{ARC}
        & \multicolumn{2}{c}{HotpotQA}
        & \multicolumn{2}{c}{QNN} \\
        \cmidrule(lr){2-3}\cmidrule(lr){4-5}\cmidrule(lr){6-7}
        \cmidrule(lr){8-9}\cmidrule(lr){10-11}\cmidrule(lr){12-13}
        \cmidrule(lr){14-15}
        \textbf{Method}
        & Max & Mean & Max & Mean & Max & Mean & Max & Mean
        & Max & Mean & Max & Mean & Max & Mean \\
        \midrule
        ProTeGi & 0.828 & 0.813 & 0.630 & 0.622 & 2.541 & 2.531 & \underline{0.632} & \underline{0.619} & \underline{0.473} & \underline{0.461} & \underline{0.441} & \underline{0.427} & \underline{0.8833} & 0.8367 \\
        LATS & 0.887 & 0.867 & 0.686 & \underline{0.682} & 2.576 & 2.542 & 0.627 & 0.609 & 0.457 & 0.432 & 0.413 & 0.409 & 0.8500 & 0.8300 \\
        EvoX & 0.943 & 0.931 & 0.684 & 0.677 & \underline{2.593} & \underline{2.583} & 0.626 & 0.617 & 0.461 & 0.443 & 0.386 & 0.357 & 0.8167 & 0.8067 \\
        EvoX (official) & \underline{0.956} & \underline{0.941} & \underline{0.696} & 0.681 & 2.592 & 2.579 & 0.613 & 0.597 & \underline{0.473} & 0.451 & 0.397 & 0.363 & 0.8333 & 0.8133 \\
        AdaEvolve & 0.769 & 0.744 & 0.619 & 0.613 & 2.573 & 2.569 & 0.621 & \underline{0.619} & 0.413 & 0.406 & 0.397 & 0.339 & 0.8333 & 0.8200 \\
        AdaEvolve (official) & 0.769 & 0.744 & 0.619 & 0.613 & 2.573 & 2.569 & 0.619 & 0.616 & 0.423 & 0.409 & 0.409 & 0.359 & 0.8500 & \underline{0.8400} \\
        GEPA & 0.719 & 0.701 & 0.659 & 0.651 & 2.437 & 2.421 & 0.613 & 0.607 & 0.396 & 0.371 & 0.381 & 0.313 & 0.8000 & 0.7900 \\
        \midrule
        \textbf{\ours{}} & \textbf{0.961} & \textbf{0.953} & \textbf{0.698} & \textbf{0.696} & \textbf{2.654} & \textbf{2.616} & \textbf{0.693} & \textbf{0.683} & \textbf{0.503} & \textbf{0.491} & \textbf{0.513} & \textbf{0.501} & \textbf{0.9000} & \textbf{0.8967} \\
        \midrule
        Relative gain & \textcolor{gainred}{+0.52\%} & \textcolor{gainred}{+1.28\%} & \textcolor{gainred}{+0.29\%} & \textcolor{gainred}{+2.05\%} & \textcolor{gainred}{+2.35\%} & \textcolor{gainred}{+1.28\%} & \textcolor{gainred}{+9.65\%} & \textcolor{gainred}{+10.34\%} & \textcolor{gainred}{+6.34\%} & \textcolor{gainred}{+6.51\%} & \textcolor{gainred}{+16.33\%} & \textcolor{gainred}{+17.33\%} & \textcolor{gainred}{+1.89\%} & \textcolor{gainred}{+6.75\%} \\
        \bottomrule
    \end{tabular}%
    }
    \endgroup
\end{table*}

\subsection{Main Results}
\label{sec:main-results}

\textbf{Overall performance.}
Figure~\ref{fig:overall_results}(a) shows that \ours{} obtains the lowest average rank among the 14 evaluated configurations across 32 benchmark groups: 1.72 for five-run Max and 2.08 for five-run Mean, compared with 6.33 and 6.35 for the strongest competing profile, EvoX. Under Max, \ours{} attains the highest score on 23 groups, including two ties. Its leading position under both summaries indicates that the aggregate advantage is not confined to best-of-five performance. Table~\ref{tab:main-results} reports the highest Max and Mean for \ours{} on all seven representative benchmarks, including comparisons with the two official-code entries. Relative gains over the strongest baseline in each column range from 0.29\% to 16.33\% for Max and from 1.28\% to 17.33\% for Mean. These results concern the evaluated configurations under the reported iteration limit; the official-code comparisons cover two methods on seven benchmarks.

\begin{figure}[t]
\vspace{-6mm}
    \centering
    \includegraphics[width=\textwidth]{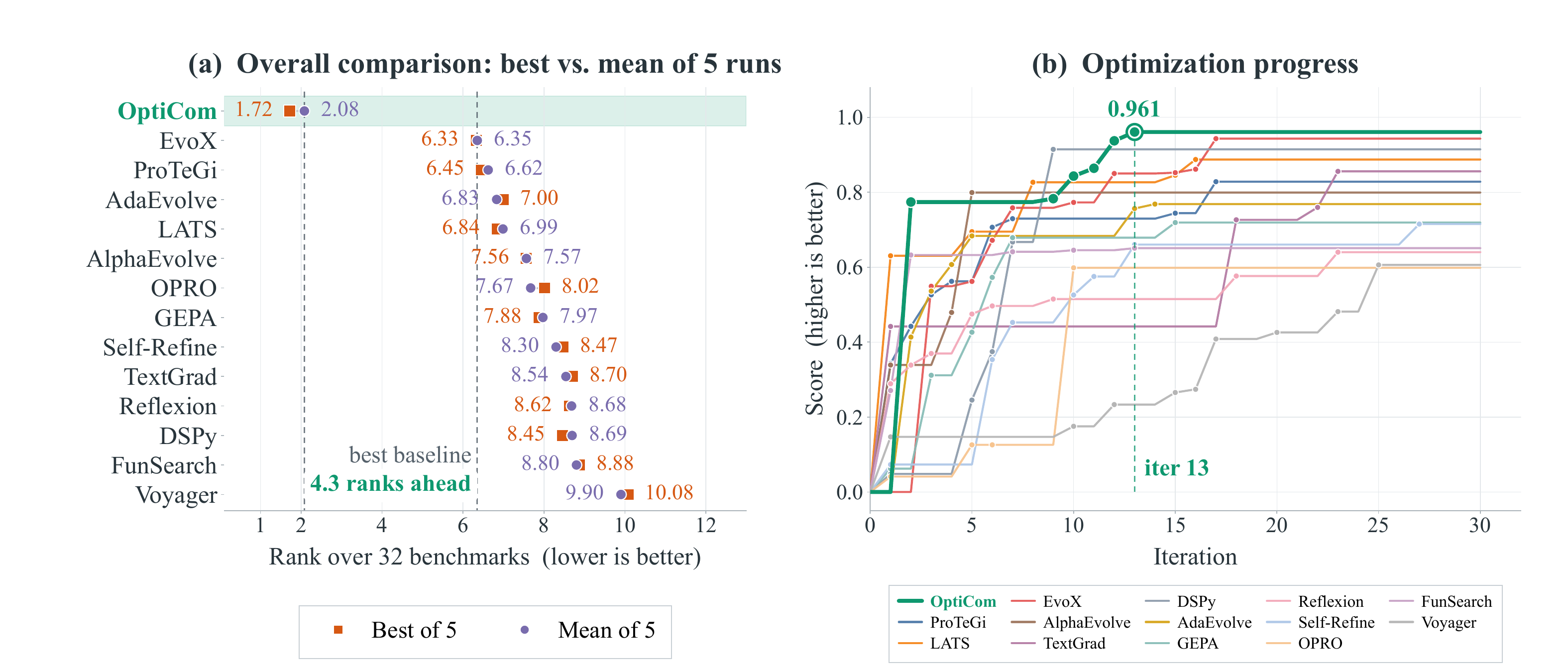}
    \vspace{-5mm}
    \caption{
        \textbf{(a)} Average within-group ranks of \ours{} and 13 method-inspired profiles across 32 benchmark groups, computed separately from five-run Max and Mean scores (lower is better).
        \textbf{(b)} Archived best-so-far scores over optimization iterations on Heilbronn Triangle.
    }
    \label{fig:overall_results}
\vspace{-6mm}
\end{figure}

\textbf{Optimization progress.}
Figure~\ref{fig:overall_results}(b) illustrates how improvements accumulate during a run. On Heilbronn Triangle, \ours{} reaches approximately 0.774 at iteration 2, makes limited progress through iteration 8, and then improves to 0.961 at iteration 13. This score exceeds the final scores of the displayed baseline profiles and remains unchanged through iteration 30. The trajectory shows that a temporary plateau need not indicate exhausted improvement opportunities. Additional trajectories reveal task-dependent refinement horizons and cases where specialized profiles finish higher (Appendix~\ref{app:optimization-trajectories}). These observations are consistent with the cumulative perspective in Section~\ref{sec:theory}, but score histories alone do not identify the effects of individual mechanism changes or establish equal-cost efficiency.

\subsection{Ablation Studies}
\label{sec:ablation-studies}

Table~\ref{tab:core-ablations} examines the roles of immediate selection, slower adaptation, and experience reuse. Removing the Strategy Adapter while retaining the state-conditioned Controller lowers mean scores by 0.240 on Heilbronn Triangle and 0.050 on HotpotQA. With the Adapter disabled in both variants, allowing the Controller to adjust Q/E/M and other action settings in addition to O improves mean scores over Operator-only adaptation by 0.016 and 0.023, respectively. Removing experience memory lowers the full framework's means by 0.060 and 0.018. Full \ours{} has the highest observed mean and lowest standard deviation among these core variants. Together, the comparisons support the benefits of broader action selection, strategy adaptation, and experience reuse in the tested settings. Full versus Fixed or Random composition changes both immediate selection and slower adaptation, so those contrasts do not isolate the Controller.

\begin{wraptable}{r}{0.46\textwidth}
\vspace{-4mm}
\begin{minipage}{\linewidth}
\centering
\scriptsize
\setlength{\tabcolsep}{2pt}
\renewcommand{\arraystretch}{1.12}
\caption{Core ablations with Doubao-Seed-2.0-pro. Scores report mean $\pm$ std over 5 runs.}
\label{tab:core-ablations}
\resizebox{\linewidth}{!}{%
\begin{tabular}{@{}lcc@{}}
\toprule
\textbf{Config.} & \textbf{Heilbronn} & \textbf{HotpotQA} \\
\midrule
Full \ours{} & $\mathbf{0.953 \pm 0.008}$ & $\mathbf{0.501 \pm 0.017}$ \\
Fixed composition & $0.704 \pm 0.034$ & $0.409 \pm 0.127$ \\
Random composition & $0.672 \pm 0.171$ & $0.381 \pm 0.153$ \\
w/o Strat. Adapter & $0.713 \pm 0.027$ & $0.451 \pm 0.058$ \\
Operator-only adapt. & $0.697 \pm 0.056$ & $0.428 \pm 0.076$ \\
w/o Exp. Memory & $0.893 \pm 0.024$ & $0.483 \pm 0.039$ \\
\bottomrule
\end{tabular}%
}
\end{minipage}
\vspace{-3mm}
\end{wraptable}

\textbf{Adaptation frequency and initialization robustness.}
On Heilbronn Triangle, event-triggered adaptation reaches a mean score of 0.953 with 10 Adapter calls, compared with 0.954 and 29 calls for every-iteration adaptation (Appendix~\ref{app:ablation-results}). This is a 65.5\% reduction in Adapter calls, while total API tokens decrease only from 877K to 865K (approximately 1.37\%). The result supports similar observed quality with fewer adaptation calls, rather than a substantial reduction in total computation. Across three initial compositions, adaptive mean scores range from 0.951 to 0.956, whereas fixed-composition means range from 0.813 to 0.857. This smaller range suggests reduced sensitivity to the tested initial choices; it does not establish an optimality bound or statistical equivalence across initializations.

\textbf{Cross-backbone robustness and model elevation.}
Table~\ref{tab:cross-backbone} evaluates five diverse foundational models under identical optimization constraints. The results establish the universal applicability of \ours{}: across every evaluated backbone--task pair, enabling the Strategy Adapter yields substantial mean-score improvements and strictly compresses variance compared to the static configuration. This consistent delta yields two critical insights. First, the framework's adaptive guidance is fundamentally robust across different model alignments and pre-training distributions; it is not overfitted to the idiosyncrasies of a single LLM. Second, the Adapter operates simultaneously as a performance ``floor raiser'' for standard models (e.g., driving a remarkable +0.240 gain for Doubao on Heilbronn) and a ``ceiling breaker'' for frontier models. For instance, equipping GPT-5.5 with Full \ours{} pushes its Heilbronn performance from a static 0.796 to a near-optimal 0.987, while collapsing the standard deviation to a mere $\pm 0.002$. Ultimately, this demonstrates that dynamic strategy adaptation structurally stabilizes the search process, acting as an indispensable performance multiplier regardless of the underlying model's raw capabilities.

\begin{figure}[t]
    \vspace{-5mm}
    \centering
    \includegraphics[width=\linewidth]{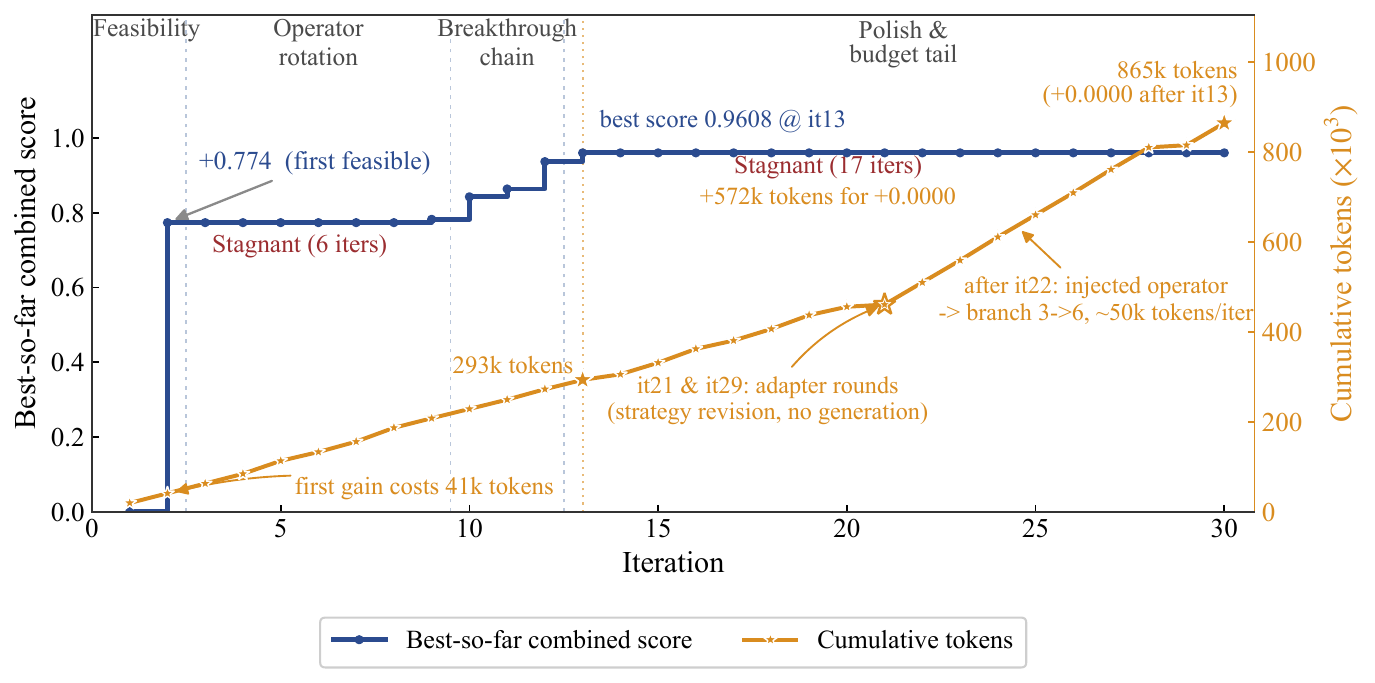}
    \vspace{-5mm}
    \caption{Heilbronn Triangle case study: best-so-far combined score (blue) and recorded cumulative API tokens (orange). The run reaches 0.9608 at iteration 13 using $\approx$293k tokens, then spends another 572k through iteration 30 without improvement. Selected execution events are annotated; their causal effects and the impossibility of further improvement are not established.}
    \label{fig:heilbronn-case-study}
    \vspace{-2mm}
\end{figure}

\begin{table}[t]
    \centering
    \small
    \setlength{\tabcolsep}{4pt}
    \renewcommand{\arraystretch}{1.12}
    \caption{Cross-backbone evaluation: mean $\pm$ SD over five runs (at most 30 iterations). Gain denotes improvement from enabling the Adapter. Bold and underlining mark the best and second-best distinct values per column, separately for means ($\uparrow$), SDs ($\downarrow$), and gains ($\uparrow$).}
    \label{tab:cross-backbone}
    \resizebox{\linewidth}{!}{%
    \begin{tabular}{@{}lccc@{\hspace{10pt}}ccc@{}}
        \toprule
        & \multicolumn{3}{c}{\textbf{Heilbronn Triangle}}
        & \multicolumn{3}{c}{\textbf{HotpotQA}} \\
        \cmidrule(lr){2-4}
        \cmidrule(lr){5-7}
        Backbone
        & w/o Adapter & Full \ours{} & Gain
        & w/o Adapter & Full \ours{} & Gain \\
        \midrule
        Doubao-Seed-2.0-pro
        & $0.713 \pm 0.027$
        & $0.953 \pm 0.008$
        & $\mathbf{+0.240}$
        & $0.451 \pm 0.058$
        & $0.501 \pm 0.017$
        & $+0.050$ \\
        GPT-5.5
        & $\mathbf{0.796} \pm \mathbf{0.011}$
        & $\mathbf{0.987} \pm \mathbf{0.002}$
        & $\underline{+0.191}$
        & $\mathbf{0.471} \pm \mathbf{0.037}$
        & $\underline{0.553} \pm \underline{0.012}$
        & $\underline{+0.082}$ \\
        Claude Opus 4.6
        & $0.787 \pm \underline{0.012}$
        & $0.973 \pm 0.007$
        & $+0.186$
        & $\underline{0.469} \pm \underline{0.041}$
        & $\mathbf{0.554} \pm \mathbf{0.010}$
        & $\mathbf{+0.085}$ \\
        GLM-5.3
        & $0.783 \pm 0.014$
        & $0.969 \pm 0.008$
        & $+0.186$
        & $0.461 \pm 0.059$
        & $0.523 \pm 0.019$
        & $+0.062$ \\
        Kimi-K3
        & $\underline{0.791} \pm \mathbf{0.011}$
        & $\underline{0.974} \pm \underline{0.004}$
        & $+0.183$
        & $0.465 \pm 0.053$
        & $0.547 \pm 0.016$
        & $\underline{+0.082}$ \\
        \bottomrule
    \end{tabular}%
    }
\end{table}

\subsection{Token-Aware Optimization Case Study}
\label{sec:test-time-efficiency}

Figure~\ref{fig:heilbronn-case-study} details a token-aware Heilbronn Triangle trajectory, exposing the stark nonlinearities of test-time compute scaling. After finding an initial feasible solution (0.774) at iteration 2, the optimizer endures a six-iteration plateau before surging to a peak score of 0.9608 at iteration 13. This sequence yields a crucial insight: \emph{early stagnation is often deceptive}. Relying on static, low-resource operators based on lack of immediate progress would forfeit significant delayed returns, underscoring the necessity of dynamically triggered exploration to break plateaus.

However, the post-peak phase reveals the severe marginal cost of unguided scaling. From iteration 13 onward, the score flatlines while cumulative token consumption explodes from roughly 293k to 865k. Driven by Adapter-triggered exploration efforts (e.g., doubling generation branch width after iteration 22), per-iteration costs soar to $\sim$50k tokens without yielding better candidates. This exposes a fundamental asymmetry in LLM optimization: while broadening search is essential to shatter early local optima, maintaining aggressive exploration at the performance frontier risks massive computational waste. This token-aware perspective directly validates \ours{}'s core premise: rather than deploying static, resource-heavy mechanisms throughout an entire run, an optimizer must dynamically modulate its execution strategy to balance breakthrough opportunities ($\Delta_t$) against their cumulative evaluation and selection costs ($\ell_t$).

\section{Conclusions and Limitations}
\label{sec:conclusion}

We introduced \ours{}, a shared AQOEMS representation and framework for state-conditioned composition of optimization mechanisms. Evaluations across 32 benchmark groups and targeted ablations support its effectiveness under the tested iteration limits. The token-aware case study highlights the cost of continued search without further improvement, while task-specific evaluation limitations and incomplete run-level records constrain the conclusions.
\newpage
\medskip

\subsection*{AI use statement}

We have not used generative AI tools for any of the tasks whose disclosure is
required---proposing or refining hypotheses, developing conceptual frameworks,
designing methodology or experiments, implementing methods, processing or
cleaning data, analysing or interpreting results, formulating or proving
mathematical claims, or translation---all of which were carried out entirely by
the authors. We did use generative AI tools for two tasks whose disclosure is
recommended: creating and refining the figures in this paper, and copy-editing
the manuscript to improve grammar and readability. All AI-assisted figures were
verified against the raw experimental outputs by the authors, and all
AI-suggested text edits were reviewed by an author and restricted to wording,
with no change to the technical content, claims, or conclusions. We have
reviewed all AI-assisted work and take responsibility for the final content of
this work, including text, claims, and artifacts produced with the aid of
generative AI.

\bibliography{iclr2027_conference}
\bibliographystyle{iclr2027_conference}

\newpage
\medskip

\appendix

\section{Optimizer Mapping and Implementation Alignment}
\label{app:optimizer-mapping}

\subsection{Expanded Related Frameworks}
\label{app:expanded-related-frameworks}

To fully contextualize \ours{}, we analyze its relationship with recent large-scale LLM optimization frameworks and platforms, explicitly delineating our unique contributions against each:

\textbf{LLM4AD (Large Language Models for Algorithm Design)~\cite{liu2024llm4ad}:} 
LLM4AD focuses on automating the discovery of programmatic algorithms using LLMs. It typically employs fixed evolutionary search pipelines (e.g., Evolution of Heuristics) to iteratively generate, evaluate, and mutate code-based solutions. 
\emph{Difference:} While LLM4AD is domain-centric (algorithm design) and prescribes specific, hardcoded evolutionary search topologies, \ours{} is a general-purpose, state-conditioned framework. Rather than executing a static pipeline, \ours{} dynamically alters its foundational search topology—fluidly switching from evolutionary recombination to error-guided local revision—based on real-time optimization state and stagnation signals.

\textbf{SkyDiscover~\cite{liu2026skydiscover}:} 
SkyDiscover serves as a comprehensive, multi-domain benchmarking platform designed to standardize the evaluation of LLM optimizers across diverse fields (e.g., mathematics, systems, quantum circuits). It provides execution harnesses and static baseline implementations of various methods. 
\emph{Difference:} SkyDiscover is the \emph{evaluation environment}, whereas \ours{} is the \emph{active optimization agent}. We utilize SkyDiscover's rigorous tasks as our testbed, but fundamentally differ by introducing a dynamic Controller-Adapter architecture capable of state-conditioned mechanism routing, rather than acting as another static algorithm profile on the benchmark.

\textbf{Universal Optimization APIs (e.g., optimize\_anything~\cite{agrawal2026optimizeanything}):} 
Recent works like \texttt{optimize\_anything} propose unified APIs that treat diverse tasks (from agent architecture search to CUDA kernel generation) as general text optimization problems. They typically rely on a single, powerful LLM proposer guided by textual feedback to iteratively improve an artifact. 
\emph{Difference:} While universal APIs provide a vital abstraction for defining tasks, they often rely on a monolithic search strategy (e.g., a fixed prompt template generating candidates based on recent feedback). \ours{} subsumes such operations into its unified AQOEMS configuration space. Instead of committing to a single proposal strategy for an entire run, \ours{} treats different search mechanisms (e.g., broad tree exploration, local diagnostic querying) as composable Action Packages, transitioning between their behaviors when the trajectory demands it.

\subsection{Scope of the Contribution}
\label{app:scope-of-contribution}

The primary contribution of this work lies in formalizing the unified AQOEMS representation and providing an executable, state-conditioned composition interface (\ours{}), validated through extensive empirical evaluation across 32 benchmark groups. Rather than introducing new individual heuristics for adaptive algorithm selection, evolutionary search, reflection, or self-modification, our work provides the rigorous architectural foundation for their dynamic coordination. 

Furthermore, the theoretical decomposition presented in Section~\ref{sec:theory} is purposefully designed to explicitly distinguish between search opportunities and selection losses in LLM optimization; it specializes existing performance-difference identities to motivate our framework's design, rather than claiming a fundamentally new mathematical theorem. Finally, the framework's extensibility ensures that new query, operator, evaluation, and memory mechanisms can be seamlessly registered via strict task contracts, though minor structural alignments may be required to express highly idiosyncratic legacy optimizers within our shared interfaces.

\subsection{AQOEMS Responsibilities and Mapping Criteria}
\label{app:mapping-criteria}

AQOEMS provides a functional decomposition of an optimizer through
$C=(A,Q,O,E,M,S)$. These dimensions describe what is optimized, what
information is acquired, how candidates are updated, how outcomes are
evaluated, what experience is retained, and how these decisions are
coordinated. They do not prescribe six separate modules or a fixed
execution order. Functions may share an implementation, and not every
dimension requires an explicit mechanism: unused functions may take
null or minimal forms. Table~\ref{tab:aqoems} summarizes the
responsibilities and illustrative realizations of each dimension.

\begin{table*}[h]
    \centering
    \small
    \setlength{\tabcolsep}{4pt}
    \renewcommand{\arraystretch}{1.15}
    \caption{
        Functional dimensions of AQOEMS.
        The dimensions describe responsibilities rather than mandatory
        modules or sequential execution steps. The examples are possible
        realizations, not required components; their use may depend on
        the optimization state through $S$.
    }
    \label{tab:aqoems}
    \begin{tabularx}{\textwidth}{
        @{}
        p{0.11\textwidth}
        >{\raggedright\arraybackslash}p{0.19\textwidth}
        >{\raggedright\arraybackslash}X
        >{\raggedright\arraybackslash}X
        @{}
    }
        \toprule
        \textbf{Dimension} &
        \textbf{Optimization concept} &
        \textbf{Functional responsibility} &
        \textbf{Illustrative realizations} \\
        \midrule

        $A$: Artifact &
        Solution representation &
        What is optimized, and how candidates are represented. &
        Source code, numerical parameters, mathematical constructions,
        prompts, or images. \\
        \addlinespace

        $Q$: Query &
        Information acquisition &
        What evidence and context are acquired or assembled
        to inform an update. &
        Inspect selected failures, request profiling evidence,
        retrieve past attempts, or use only the current candidate
        and feedback. \\
        \addlinespace

        $O$: Operator &
        Search and update operations &
        How candidates are generated or modified. &
        Generate a new candidate, repair an error, apply a local
        revision, or recombine existing candidates. \\
        \addlinespace

        $E$: Evaluation &
        Objective and constraint assessment &
        How candidate quality and feasibility are assessed,
        and what feedback is returned. &
        Correctness checks, runtime measurements, diagnostic tests,
        or higher-fidelity verification. \\
        \addlinespace

        $M$: Memory &
        Search history and experience &
        What information persists across iterations,
        and how it is maintained. &
        Retain an incumbent, a candidate archive, evaluation records,
        or reusable lessons; omit persistent experience storage
        when unused. \\
        \addlinespace

        $S$: Strategy &
        Search control &
        How the other functions are coordinated and resources
        are allocated as optimization proceeds. &
        Follow a prescribed schedule, apply state-dependent rules,
        or adapt selection preferences using accumulated outcomes. \\
        \bottomrule
    \end{tabularx}
\end{table*}

\paragraph{Functional boundaries.}
The dimensions distinguish roles within an optimization process,
even when those roles share code or model calls. For example, $E$
may execute a candidate program and return correctness results,
runtime measurements, and failure records. $M$ specifies how these
records and associated candidates persist across iterations, while
$Q$ specifies how selected failures, additional measurements, or
stored experiences are obtained and assembled for a subsequent
update. $O$ defines the transformation applied to the candidate,
and $S$ coordinates which information-acquisition and update
mechanisms are used and with what resources. Thus, retrieving a
past attempt involves both the storage functionality of $M$ and
the information-access functionality of $Q$. Likewise, $Q$ may
request additional measurements through procedures defined by $E$.
These interactions do not require separate implementations for
each responsibility.

\paragraph{State-dependent realization.}
The mechanisms used to fulfill these responsibilities may depend
on the current candidate, unresolved failures, accumulated experience,
and remaining budget. For instance, an invalid program may trigger
failure inspection and repair, whereas a valid but slow program may
trigger profiling and a targeted performance revision. This does not
require every dimension to change at every iteration: the artifact
representation may remain fixed while context selection, update
operators, and evaluation depth vary. Such variation must remain
consistent with the task objective and feasibility requirements;
changing an evaluation procedure does not by itself redefine the
optimization problem.

\paragraph{Configurations and execution policies.}
A configuration specifies mechanisms and coordinating rules rather
than a predetermined sequence of actions. In particular, $S$ may
include conditional decisions or internal adaptation based on the
optimization state. Representing an existing optimizer as one point
in the configuration space can express its adaptive behavior; a mapping alone does not verify that an experimental profile reproduces it.
Holding that configuration fixed means retaining its defining
mechanisms and rules, not forcing identical actions across iterations.
The reference policy in Section~\ref{sec:theory} uses this distinction: it may contain state-dependent or internally adaptive rules. The theory does not equate a fixed configuration with a constant action.

\paragraph{Mapping criteria.}
We map existing systems according to their implemented functionality,
rather than their original module names. A mapping identifies the
candidate representation, information sources, update operations,
evaluation procedures, persistent state, and coordinating rules.
Where a responsibility is implicit, shared with another component,
or absent, the mapping records that fact instead of introducing an
additional mechanism. Internal adaptation is included in $S$, and
task-specific assumptions are retained. A conceptual mapping alone
does not establish implementation equivalence: any changes to prompts,
operators, feedback access, or resource accounting must be documented
when instantiating a method within the framework. Complete mappings
are provided in Appendix~\ref{app:complete-mappings}, with
implementation differences and alignment checks in
Appendix~\ref{app:alignment-checks}.

\subsection{Complete AQOEMS Mappings}
\label{app:complete-mappings}

Table~\ref{tab:complete-optimizer-mapping} provides functional mappings
for the methods associated with our baseline configurations and for
two additional conceptual references, Trace and metaTextGrad.
The mappings describe the published mechanisms; they are not, by
themselves, claims that every mechanism is reproduced in our
experimental implementation.
The correspondence between these mechanisms and the registered
implementations is addressed in Appendix~\ref{app:alignment-checks}.

We distinguish methods with a corresponding registered configuration
(\textbf{B}) from conceptual mappings only (\textbf{C}).
The B/C annotations distinguish the 13 method-inspired baseline profiles evaluated here from conceptual references; conceptual entries are not empirical comparisons. For framework-level entries, the mapping describes
a family of configurations whose concrete behavior depends on the
selected optimizer.

\begingroup
\small
\setlength{\tabcolsep}{4pt}
\setlength{\LTcapwidth}{\textwidth}
\renewcommand{\arraystretch}{1.12}

\begin{longtable}{
    @{}
    >{\raggedright\arraybackslash}p{0.17\textwidth}
    >{\raggedright\arraybackslash}p{0.385\textwidth}
    >{\raggedright\arraybackslash}p{\dimexpr0.445\textwidth-16pt\relax}
    @{}
}
    \caption{
        Complete functional mappings. B denotes a corresponding baseline profile in the supplied registry; C denotes conceptual coverage only. B does not certify equivalence to the original implementation. The two right columns jointly specify all six AQOEMS dimensions.
    }
    \label{tab:complete-optimizer-mapping}\\
    \toprule
    \textbf{Method / status} &
    \textbf{Artifact, query, and operator} &
    \textbf{Evaluation, memory, and strategy} \\
    \midrule
    \endfirsthead

    \multicolumn{3}{l}{
        \tablename~\thetable\ continued
    }\\
    \toprule
    \textbf{Method / status} &
    \textbf{Artifact, query, and operator} &
    \textbf{Evaluation, memory, and strategy} \\
    \midrule
    \endhead

    \midrule
    \multicolumn{3}{r}{Continued on next page}\\
    \endfoot

    \bottomrule
    \endlastfoot

    OPRO~\cite{yang2024large}
    \newline \textbf{B}
    &
    \textbf{A:} Candidate solutions, including prompts.
    \newline
    \textbf{Q:} Task description and selected scored solutions.
    \newline
    \textbf{O:} Generate proposals conditioned on that history.
    &
    \textbf{E:} Task objective values.
    \newline
    \textbf{M:} Evaluated candidate--score records.
    \newline
    \textbf{S:} Repeated generation and evaluation with
    history selection and stopping rules.
    \\
    \addlinespace

    TextGrad~\cite{yuksekgonul2025optimizing}
    \newline \textbf{B}
    &
    \textbf{A:} Optimizable variables in a computational graph.
    \newline
    \textbf{Q:} Relevant graph context and propagated feedback.
    \newline
    \textbf{O:} Variable updates guided by textual feedback.
    &
    \textbf{E:} Objective evaluation and associated criticism.
    \newline
    \textbf{M:} Graph state, current variables, and
    optimizer-dependent update history.
    \newline
    \textbf{S:} Feedback propagation and variable-update procedure.
    \\
    \addlinespace

    ProTeGi~\cite{pryzant2023automatic}
    \newline \textbf{B}
    &
    \textbf{A:} Task prompts.
    \newline
    \textbf{Q:} Minibatch examples and observed errors.
    \newline
    \textbf{O:} Generate textual critiques and corresponding
    prompt revisions.
    &
    \textbf{E:} Predictive performance on evaluation examples.
    \newline
    \textbf{M:} Beam candidates and evaluation statistics.
    \newline
    \textbf{S:} Beam expansion and bandit-based candidate selection.
    \\
    \addlinespace

    Reflexion~\cite{shinn2023reflexion}
    \newline \textbf{B}
    &
    \textbf{A:} Task attempts, including programs or action trajectories.
    \newline
    \textbf{Q:} Current feedback and previous verbal reflections.
    \newline
    \textbf{O:} Produce a subsequent attempt informed by reflection.
    &
    \textbf{E:} Task-dependent external or internal feedback.
    \newline
    \textbf{M:} Episodic verbal reflections.
    \newline
    \textbf{S:} Attempt, evaluate, reflect, and retry.
    \\
    \addlinespace

    GEPA~\cite{agrawal2025gepa}
    \newline \textbf{B}
    &
    \textbf{A:} Prompts in an AI system.
    \newline
    \textbf{Q:} Execution trajectories, diagnostics,
    and selected candidate context.
    \newline
    \textbf{O:} Reflective mutation and combination of
    complementary candidates.
    &
    \textbf{E:} Task evaluations, including performance across examples.
    \newline
    \textbf{M:} Candidate pool, scores, and ancestry.
    \newline
    \textbf{S:} Pareto-aware selection and evolutionary updates
    under an evaluation budget.
    \\
    \addlinespace

    FunSearch~\cite{FunSearch2023}
    \newline \textbf{B}
    &
    \textbf{A:} Functions within a program specification.
    \newline
    \textbf{Q:} Selected prior functions assembled into a prompt.
    \newline
    \textbf{O:} LLM-generated function variants.
    &
    \textbf{E:} Automated execution and scoring.
    \newline
    \textbf{M:} An island-structured database of evaluated programs.
    \newline
    \textbf{S:} Program sampling, insertion, and island reset rules.
    \\
    \addlinespace

    AlphaEvolve~\cite{novikov2025alphaevolvecodingagentscientific}
    \newline \textbf{B}
    &
    \textbf{A:} Programs implementing candidate solutions.
    \newline
    \textbf{Q:} Selected programs, evaluation feedback,
    and task context.
    \newline
    \textbf{O:} LLM-proposed code changes.
    &
    \textbf{E:} Automated evaluators, potentially organized in stages.
    \newline
    \textbf{M:} Evaluated program database and associated metadata.
    \newline
    \textbf{S:} Evolutionary sampling, evaluation,
    and population management.
    \\
    \addlinespace

    AdaEvolve~\cite{cemri2026adaevolveadaptivellmdriven}
    \newline \textbf{B}
    &
    \textbf{A:} Candidate programs.
    \newline
    \textbf{Q:} Selected program context and improvement evidence.
    \newline
    \textbf{O:} LLM-generated program variations.
    &
    \textbf{E:} Program fitness.
    \newline
    \textbf{M:} Island populations and accumulated improvement statistics.
    \newline
    \textbf{S:} Adaptive exploration intensity, inter-island
    resource allocation, and meta-guidance.
    \\
    \addlinespace

    EvoX~\cite{liu2026evoxmetaevolutionautomateddiscovery}
    \newline \textbf{B}
    &
    \textbf{A:} Candidate programs; search procedures
    at the meta level.
    \newline
    \textbf{Q:} Candidate history and evidence about search performance.
    \newline
    \textbf{O:} Program variation and meta-level modification
    of search procedures.
    &
    \textbf{E:} Candidate quality and the performance
    induced by search strategies.
    \newline
    \textbf{M:} Candidate and strategy histories.
    \newline
    \textbf{S:} Coupled evolution of solutions and
    their generation and management procedures.
    \\
    \addlinespace

    Voyager~\cite{wang2023voyager}
    \newline \textbf{B}
    &
    \textbf{A:} Executable skill programs.
    \newline
    \textbf{Q:} Environment observations, errors,
    and retrieved skills.
    \newline
    \textbf{O:} Skill generation and iterative program revision.
    &
    \textbf{E:} Environment feedback and task-success verification.
    \newline
    \textbf{M:} A reusable skill library.
    \newline
    \textbf{S:} Automatic curriculum and iterative skill acquisition.
    \\
    \addlinespace

    DSPy~\cite{khattab2024dspy}
    \newline \textbf{B}
    &
    \textbf{A:} Parameters of an LM program,
    such as instructions and demonstrations.
    \newline
    \textbf{Q:} Training examples and optimizer-dependent traces.
    \newline
    \textbf{O:} Updates supplied by the selected DSPy optimizer.
    &
    \textbf{E:} A user-specified task metric.
    \newline
    \textbf{M:} Optimizer-dependent candidate, demonstration,
    and search state.
    \newline
    \textbf{S:} The selected compilation and search procedure;
    DSPy does not specify a single universal optimizer.
    \\
    \addlinespace

    Self-Refine~\cite{madaan2023selfrefine}
    \newline \textbf{B}
    &
    \textbf{A:} Generated outputs.
    \newline
    \textbf{Q:} The current output and feedback context.
    \newline
    \textbf{O:} Revision using self-generated feedback.
    &
    \textbf{E:} LLM-generated assessment of the output.
    \newline
    \textbf{M:} Outputs and feedback retained in the refinement context.
    \newline
    \textbf{S:} Alternating feedback and refinement
    until a stopping condition is met.
    \\
    \addlinespace

    LATS~\cite{zhou2024lats}
    \newline \textbf{B}
    &
    \textbf{A:} Partial and complete action or reasoning trajectories.
    \newline
    \textbf{Q:} Environment observations and reflections.
    \newline
    \textbf{O:} Generate and expand candidate trajectory continuations.
    &
    \textbf{E:} Environment rewards and LM-based value estimates.
    \newline
    \textbf{M:} Search-tree statistics and reflection history.
    \newline
    \textbf{S:} Monte Carlo tree search with selection,
    expansion, evaluation, and backpropagation.
    \\
    \addlinespace

    Trace / OptoPrime~\cite{trace2024}
    \newline \textbf{C}
    &
    \textbf{A:} Optimizable parameters of computational workflows.
    \newline
    \textbf{Q:} Execution traces and associated feedback.
    \newline
    \textbf{O:} Updates supplied by OptoPrime or another
    optimizer connected to Trace.
    &
    \textbf{E:} Workflow feedback supplied through a trace oracle.
    \newline
    \textbf{M:} Workflow state and optimizer-dependent history.
    \newline
    \textbf{S:} The selected optimizer's update procedure.
    \\
    \addlinespace

    metaTextGrad~\cite{metatextgrad2025}
    \newline \textbf{C}
    &
    \textbf{A:} Optimizer prompts and compositions at the outer level.
    \newline
    \textbf{Q:} Reference optimizers and validation evidence.
    \newline
    \textbf{O:} Prompt refinement and structural composition
    of optimizers.
    &
    \textbf{E:} Downstream validation performance after optimization.
    \newline
    \textbf{M:} Reference and previously evaluated optimizer candidates.
    \newline
    \textbf{S:} Meta-level proposal, evaluation, and selection,
    with separate prompt and structure optimization.
    \\

\end{longtable}
\endgroup

\paragraph{Frameworks and nested optimization.}
A framework such as DSPy or Trace corresponds to multiple possible
configurations because its update rules depend on the optimizer
attached to it. Naming the framework alone is therefore insufficient
to identify an experimental baseline.
Trace supplies an execution-trace abstraction for generative
optimization, while AQOEMS describes the responsibilities that
constitute the optimizer operating on such information.
For metaTextGrad, the mapping in
Table~\ref{tab:complete-optimizer-mapping} is made at the outer level,
where optimizers are the artifacts being improved.
Its inner optimization process can be represented by a separate
AQOEMS configuration.
The same distinction applies when EvoX modifies the procedures used
by its candidate search. These nested representations preserve the
level at which each decision is made.

\paragraph{Task adaptation.}
The artifact and evaluator in an experimental configuration may differ
from those in the source method's original application.
For example, transferring a prompt-update mechanism to program
optimization changes the artifact representation and evaluation
interface. Such an instance is a task-adapted configuration, whose
relationship to the source method depends on the mechanisms retained.
Similarly, representing Voyager's skill-improvement loop does not
make a bounded optimization benchmark equivalent to its original
open-ended environment.
The shared representation exposes these adaptations so that they
can be reported explicitly.

\paragraph{Memory and strategy are functional, not nominal.}
A population, search tree, beam, or in-prompt candidate history
constitutes memory even when the original implementation has no
module named ``memory.'' Likewise, a method with no explicit
strategy-adaptation module still has coordinating rules in $S$.
Conversely, an adaptive $S$ can be part of one fixed configuration.
These conventions prevent implementation labels such as
``no memory'' or ``no strategy'' from obscuring state and control
that are present elsewhere in the execution process.

\subsection{Implementation Differences and Alignment Checks}
\label{app:alignment-checks}

The mappings in Appendix~\ref{app:complete-mappings} describe how existing optimizers can be instantiated within the shared configuration space. A functional mapping identifies the responsibilities of an optimizer, but does not by itself establish equivalence to its original implementation. We therefore distinguish the mechanism that defines a method from the interfaces used to execute it. For example, an evolutionary method is characterized by how it selects and combines candidates, whereas a feedback-guided method is characterized by how feedback informs subsequent revisions. These mechanisms can operate through shared candidate, evaluation, and memory interfaces without requiring identical prompts, storage formats, or execution infrastructure.

An instantiation should preserve the dependencies that make its defining mechanism effective. Retaining a reflection module, for instance, is insufficient if its output is never supplied to subsequent attempts; similarly, retaining a candidate archive is insufficient if the selection procedure does not use it. Table~\ref{tab:instantiation-checks} organizes alignment checks around such observable dependencies. Components may be stateful or randomized, and a fixed configuration may contain an adaptive strategy. Alignment therefore concerns the method's decision rules and information flow, rather than requiring the same action at every iteration.

\begin{table}[ht]
    \centering
    \small
    \setlength{\tabcolsep}{4pt}
    \renewcommand{\arraystretch}{1.15}
    \caption{Criteria for checking framework instantiations. These checks distinguish preservation of a method's defining mechanism from changes introduced by shared interfaces. The table specifies verification criteria rather than certifying equivalence to official implementations.}
    \label{tab:instantiation-checks}
    \begin{tabular}{
        @{}
        >{\raggedright\arraybackslash}p{0.19\textwidth}
        >{\raggedright\arraybackslash}p{0.35\textwidth}
        >{\raggedright\arraybackslash}p{\dimexpr0.46\textwidth-16pt\relax}
        @{}
    }
        \toprule
        \textbf{Aspect} &
        \textbf{Mechanism to preserve} &
        \textbf{Alignment check} \\
        \midrule
        Candidate generation &
        The information and update rule used to produce new candidates. &
        Check that generation receives the intended parents, examples, feedback, or revision instructions. \\

        Selection and search &
        The method's candidate selection, branching, population, or tree-search rule. &
        Check that these rules affect executed actions, rather than appearing only in configuration descriptions. \\

        Feedback use &
        The feedback representation and its role in subsequent optimization. &
        Check that evaluation records, reflections, or textual critiques reach the decisions they are intended to guide. \\

        Memory and reuse &
        The information retained across iterations and the rule for retrieving it. &
        Check that retained candidates or experiences remain available and are retrieved according to the configured mechanism. \\

        Strategy adaptation &
        The trigger, evidence, and scope of changes to the optimization procedure. &
        Check that strategy updates affect later decisions; distinguish adaptive rules within a configuration from changes to the configuration itself. \\

        Evaluation and budget &
        Task validity requirements, scoring semantics, and resource constraints. &
        Check that candidates use the shared benchmark evaluator and that reported limits include the relevant optimization operations. \\
        \bottomrule
    \end{tabular}
\end{table}

Implementation differences should be documented at the level at which they can affect behavior. Relevant differences include prompt templates, population or branch sizes, retrieval rules, adaptation schedules, stopping conditions, and access to external tools. A component that is omitted or replaced should be identified explicitly rather than inferred from the method name. Likewise, translating a method to a new artifact domain may require a different proposal template while retaining its selection and feedback mechanisms. Comparisons in this paper therefore concern the evaluated framework instantiations, rather than unrestricted claims about every implementation of the corresponding methods.

This distinction also separates optimizer alignment from benchmark alignment. As described in Section~\ref{sec:experimental-setup}, the benchmark evaluation pipelines are integrated from SkyDiscover, with study-specific aggregation stated explicitly in Appendix~\ref{app:benchmark-details}. Sharing this evaluation pipeline makes candidate outcomes comparable; it does not alone establish that each optimizer reproduces all details of its official implementation. The common representation provides a basis for making these retained mechanisms and implementation differences explicit.

\section{Motivation Study: Protocol and Supplementary Results}
\label{app:motivation-protocol}

\subsection{Task, States, and Reference Configuration}
\label{app:motivation-setup}

We use the Heilbronn triangle task in \texttt{benchmarks/math/heilbronn\_triangle}. A candidate program implements \texttt{heilbronn\_triangle11()} and returns an array of shape $(11,2)$ containing points within or on the boundary of the equilateral triangle $\mathcal{T}$ with vertices $(0,0)$, $(1,0)$, and $(1/2,\sqrt{3}/2)$. The objective is to maximize the smallest triangle area among all triples of points:
\begin{equation}
    \max_{p_1,\ldots,p_{11}\in\mathcal{T}}
    \min_{1\leq i<j<k\leq11}
    \frac{\operatorname{Area}(p_i,p_j,p_k)}
         {\operatorname{Area}(\mathcal{T})}.
    \label{eq:motivation-heilbronn-objective}
\end{equation}
The evaluator checks output shape and triangle containment, evaluates all $\binom{11}{3}=165$ triples, and returns the benchmark score and associated execution evidence. We distinguish these objective measurements from the binary state-resolution outcomes used in this diagnostic study and from the bounded terminal utility in Section~\ref{sec:theory}.

The reference configuration is $R=Q2$ self-only + $O1$ local revision + $E2$ textual feedback + $M1$ no persistent experience memory + $S1$ fixed strategy. In the base intervention matrix, one mechanism is changed while the remaining reference rules are retained. Each of the five state families contains three frozen starting situations, and each mechanism--situation pair is evaluated over ten trials with the same maximum ten-iteration horizon. Thus, a displayed mechanism--state rate aggregates 30 trials with equal weight across the three situations. Reference trials are shared when the reference setting reappears across dimensions; they are not counted as independent repetitions.

\subsection{State Definitions and Success Criteria}
\label{app:motivation-success}

The five state families are: s1 execution or hard-constraint failure; s2 failed correctness tests; s3 objective stagnation; s4 limited remaining budget; and s5 high performance variability. A trial is successful if it resolves the targeted state condition within the declared budget while preserving all requirements that should already hold at that stage. Table~\ref{tab:motivation-success-criteria} summarizes the criteria. GPT-5.5 judges state resolution according to explicitly defined, pre-specified criteria for the starting situations. The reported success rates therefore measure criterion-based LLM judgments of state resolution, rather than the benchmark objective score itself. Benchmark execution checks and objective measurements remain distinct from this diagnostic judgment.

\begin{table}[ht]
    \centering
    \small
    \renewcommand{\arraystretch}{1.12}
    \caption{Optimization states and state-resolution criteria in the Heilbronn diagnostic study.}
    \label{tab:motivation-success-criteria}
    \begin{tabularx}{\linewidth}{@{}l>{\raggedright\arraybackslash}p{0.25\linewidth}>{\raggedright\arraybackslash}X@{}}
        \toprule
        \textbf{State} & \textbf{Initial condition} & \textbf{Success criterion} \\
        \midrule
        s1 & Execution failure or hard-constraint violation & Recover a candidate that executes and satisfies the required hard constraints. \\
        \addlinespace
        s2 & Failed correctness tests & Preserve feasibility and pass all required correctness checks. \\
        \addlinespace
        s3 & Objective stagnation & Preserve feasibility and correctness and exceed the predefined objective-improvement criterion. \\
        \addlinespace
        s4 & Limited remaining budget & Reach the s3 objective criterion within the declared remaining budget. \\
        \addlinespace
        s5 & High performance variability & Reduce variability beyond the predefined criterion without exceeding the allowed degradation in mean quality, feasibility, or correctness. \\
        \bottomrule
    \end{tabularx}
\end{table}

\subsection{Three Situations per State}
\label{app:motivation-situations}

The three situations within each state represent different failure modes rather than repeated copies of one snapshot. All mechanisms compared within a state use the same three starting situations and the same situation-specific resource limits. The 15 situations are listed below.

\begingroup
\small
\setlength{\tabcolsep}{4pt}
\setlength{\LTcapwidth}{\textwidth}
\renewcommand{\arraystretch}{1.12}
\begin{longtable}{@{}p{0.07\textwidth}p{0.22\textwidth}p{\dimexpr0.71\textwidth-16pt\relax}@{}}
\caption{Starting situations for the expanded Heilbronn Triangle diagnostic study. Each state contains three situations with ten trials per mechanism and situation.}
\label{tab:motivation-situations}\\
\toprule
\textbf{ID} & \textbf{Situation} & \textbf{Starting condition}\\
\midrule
\endfirsthead
\multicolumn{3}{l}{\tablename~\thetable\ continued}\\
\toprule
\textbf{ID} & \textbf{Situation} & \textbf{Starting condition}\\
\midrule
\endhead
\midrule
\multicolumn{3}{r}{Continued on next page}\\
\endfoot
\bottomrule
\endlastfoot
s1-1 & Runtime exception & Point generation or updating raises an exception, such as an out-of-range index or incompatible array dimensions, and terminates without returning a point set.\\
\addlinespace
s1-2 & Construction timeout & Excessive search, too many internal iterations, or an ineffective termination condition prevents the construction from finishing within the declared execution limit.\\
\addlinespace
s1-3 & Invalid output interface & The program terminates but returns the wrong number of points, an array other than shape $(11,2)$, or coordinates containing NaN or infinity.\\
\midrule
s2-1 & Incorrect boundary handling & The program returns a well-formed point set, but rectangular clipping or an incorrect projection leaves points outside the sloping boundaries of the equilateral triangle.\\
\addlinespace
s2-2 & Incorrect area objective & The search uses an incorrect internal area calculation, such as omitting the absolute determinant or using an incorrect normalization factor, so its internal objective disagrees with the specified geometric objective.\\
\addlinespace
s2-3 & Incomplete triple enumeration & The internal objective evaluates only a subset of triples, such as consecutive or locally selected points, and can miss the triple determining the true minimum area.\\
\midrule
s3-1 & Local-perturbation plateau & A valid construction repeatedly applies small coordinate perturbations near the incumbent layout, without a qualifying improvement during the predefined observation window.\\
\addlinespace
s3-2 & Structure-restricted plateau & Search retains a fixed layout structure, such as symmetry constraints, a fixed number of boundary points, or a fixed grouping, and parameter updates fail to improve the score during the observation window.\\
\addlinespace
s3-3 & Unproductive restarts & Repeated restarts use the same initialization distribution and search rules, producing similarly scored solutions without improving the retained best score during the observation window.\\
\midrule
s4-1 & Expensive construction & The current valid candidate requires substantial internal search or local optimization; continuing that procedure consumes most of the remaining budget and leaves little room for further attempts.\\
\addlinespace
s4-2 & Too many active branches & Several exploration directions remain active, but the remaining budget cannot support continued substantial allocation to every branch under the current schedule.\\
\addlinespace
s4-3 & Budget-mismatched updates & The current procedure proposes large structural changes or complete restarts whose search and verification demands are poorly matched to the limited remaining budget.\\
\midrule
s5-1 & Initialization sensitivity & Different random initial point sets lead the same construction procedure to substantially different final scores, with strong outcomes in some executions and weak outcomes in others.\\
\addlinespace
s5-2 & Update-path sensitivity & Starting from the same initial point set, random choices of updated points, perturbation directions, or update order produce substantially different final scores.\\
\addlinespace
s5-3 & Strategy-selection sensitivity & The construction procedure randomly selects among search or construction strategies with differing reliability, producing substantial variation in final performance across executions.\\
\end{longtable}
\endgroup

\paragraph{Thirty-trial aggregation.}
For state $s$, dimension $d$, mechanism $v$, situation $c\in\{1,2,3\}$, and repetition $r\in\{1,\ldots,10\}$, let $Y_{s,d,v,c}^{(r)}$ be the binary state-resolution outcome. We report
\begin{equation}
    \widehat{p}_{s,d,v}
    =\frac{1}{30}\sum_{c=1}^{3}\sum_{r=1}^{10}Y_{s,d,v,c}^{(r)}
    =\frac{1}{3}\sum_{c=1}^{3}\widehat{p}_{s,d,v,c},
    \qquad
    \widehat{p}_{s,d,v,c}=\frac{1}{10}\sum_{r=1}^{10}Y_{s,d,v,c}^{(r)}.
    \label{eq:motivation-expanded-success-rate}
\end{equation}
Dimension-level summaries average these mechanism rates over the displayed alternatives. Because each state has a different resolution criterion, rates should be compared primarily within a state rather than interpreted as equal changes in a shared objective scale.

\subsection{Mechanisms, Memory Gating, and Complete Situation-Level Results}
\label{app:motivation-mechanisms}

The query alternatives are none, self-only, environment probing, and memory lookup. Operator alternatives are local revision, repair, and recombination. Evaluation alternatives are scalar-only, textual feedback, and staged evaluation. Strategy alternatives are fixed rules, bandit adaptation, and meta-evolution. In the strict one-dimension reference intervention, memory lookup is inert when $M1$ stores no reusable library, and episodic or structured memory is inert when $Q2$ never reads it. We therefore report the strict base matrix separately from a gate-open memory analysis: episodic, structured, and adaptive memory are paired with $Q4$ memory lookup as the reader, while the $Q4$ row itself remains evaluated under $M1$. This qualification is important because the memory panel is a conditional memory comparison rather than a literal one-coordinate intervention.

\begin{figure}[ht]
    \centering
    \includegraphics[width=\linewidth]{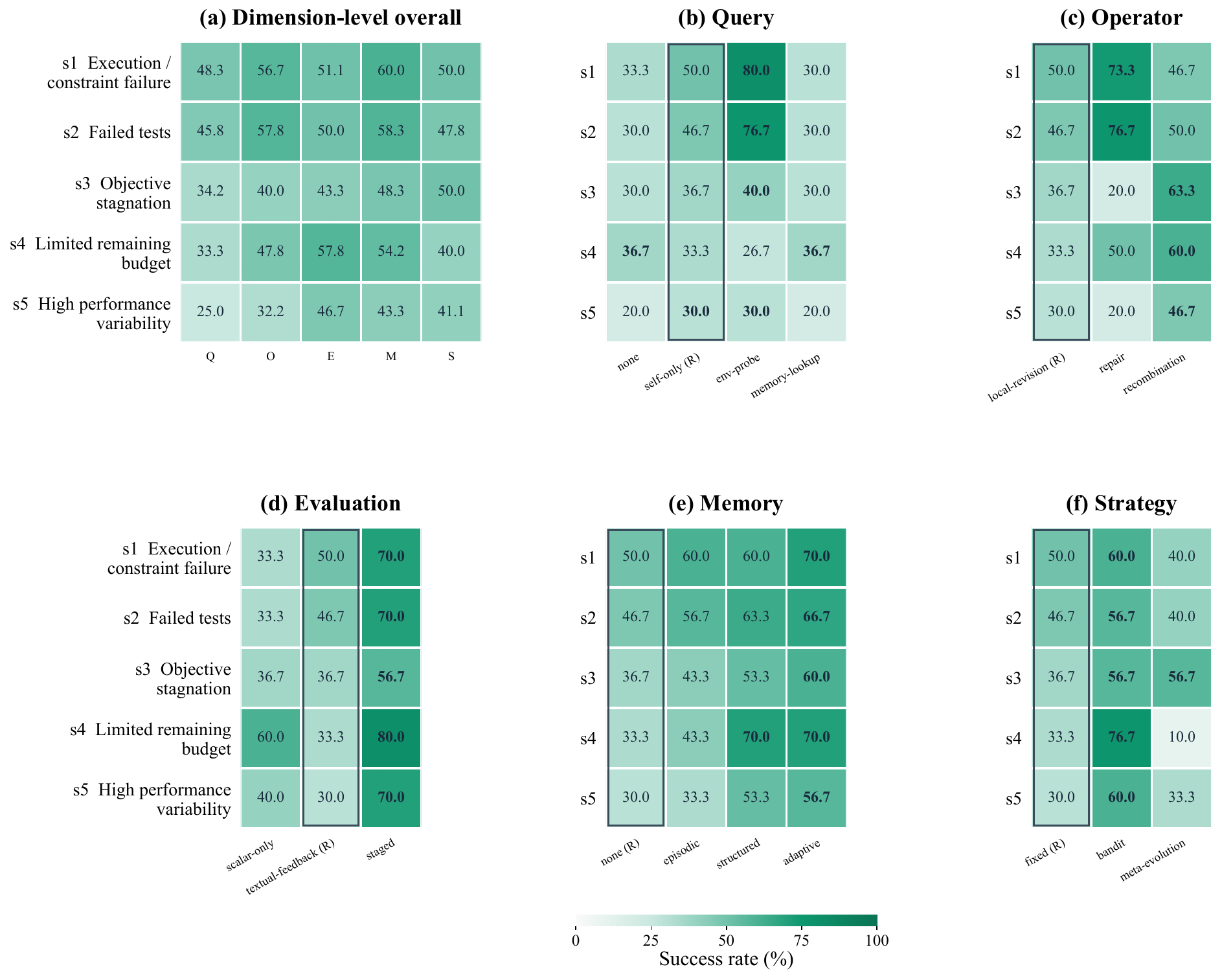}
    \caption{Complete mechanism-level view of the Heilbronn Triangle motivation study. Each displayed mechanism--state value averages the three situations in that state, with ten trials per situation. \textbf{(a)} Dimension-level averages. \textbf{(b)--(f)} Query, operator, evaluation, memory, and strategy mechanisms. R denotes the reference setting. The memory panel uses the gate-open comparison described in this appendix so that stored episodic, structured, and adaptive memory can be read by the memory-lookup query mechanism.}
    \label{fig:motivation-all}
\end{figure}

\begin{table*}[htbp]
\centering
\scriptsize
\renewcommand{\arraystretch}{1.08}
\caption{Successful trials out of ten for all 15 situations in the strict base intervention matrix. R marks the reference mechanism. Memory rows $M2$ and $M3$ are inert here because the reference query does not read persistent memory; the gate-open memory comparison is reported separately in Table~\ref{tab:motivation-memory-gated}.}
\label{tab:motivation-situation-results}
\resizebox{\textwidth}{!}{%
\begin{tabular}{@{}ll*{15}{c}@{}}
\toprule
\textbf{ID} & \textbf{Mechanism} & \textbf{s1-1} & \textbf{s1-2} & \textbf{s1-3} & \textbf{s2-1} & \textbf{s2-2} & \textbf{s2-3} & \textbf{s3-1} & \textbf{s3-2} & \textbf{s3-3} & \textbf{s4-1} & \textbf{s4-2} & \textbf{s4-3} & \textbf{s5-1} & \textbf{s5-2} & \textbf{s5-3} \\
\midrule
Q1 & none & 3&3&4&3&3&3&3&3&3&4&3&4&2&2&2\\
Q2 & self-only (R) & 5&4&6&5&5&4&4&3&4&4&3&3&3&3&3\\
Q3 & env-probe & 9&7&8&8&8&7&4&4&4&2&3&3&3&3&3\\
Q4 & memory-lookup & 3&2&4&3&3&3&3&3&3&4&3&4&2&2&2\\
\midrule
O1 & local-revision (R) & 5&4&6&5&5&4&4&3&4&4&3&3&3&3&3\\
O2 & repair & 8&6&8&8&8&7&2&2&2&5&4&6&2&2&2\\
O3 & recombination & 5&4&5&5&5&5&7&6&6&6&6&6&5&5&4\\
\midrule
E1 & scalar-only & 3&3&4&4&3&3&4&3&4&6&6&6&4&4&4\\
E2 & textual-feedback (R) & 5&4&6&5&5&4&4&3&4&4&3&3&3&3&3\\
E3 & staged & 7&6&8&7&7&7&6&5&6&8&8&8&7&7&7\\
\midrule
M1 & none (R) & 5&4&6&5&5&4&4&3&4&4&3&3&3&3&3\\
M2 & episodic & 5&4&6&5&5&4&4&3&4&4&3&3&3&3&3\\
M3 & structured & 5&4&6&5&5&4&4&3&4&4&3&3&3&3&3\\
\midrule
S1 & fixed (R) & 5&4&6&5&5&4&4&3&4&4&3&3&3&3&3\\
S2 & bandit & 6&5&7&6&6&5&6&5&6&7&9&7&5&5&8\\
S3 & meta-evolution & 4&3&5&4&4&4&5&6&6&1&1&1&3&3&4\\
\bottomrule
\end{tabular}%
}
\end{table*}

\begin{table*}[htbp]
\centering
\scriptsize
\renewcommand{\arraystretch}{1.08}
\caption{Gate-open memory comparison. Episodic, structured, and adaptive memory are paired with $Q4$ memory lookup so that stored experience can affect later decisions. Entries are successful trials out of ten.}
\label{tab:motivation-memory-gated}
\resizebox{\textwidth}{!}{%
\begin{tabular}{@{}ll*{15}{c}@{}}
\toprule
\textbf{ID} & \textbf{Mechanism} & \textbf{s1-1} & \textbf{s1-2} & \textbf{s1-3} & \textbf{s2-1} & \textbf{s2-2} & \textbf{s2-3} & \textbf{s3-1} & \textbf{s3-2} & \textbf{s3-3} & \textbf{s4-1} & \textbf{s4-2} & \textbf{s4-3} & \textbf{s5-1} & \textbf{s5-2} & \textbf{s5-3} \\
\midrule
M2 & episodic + Q4 & 6&5&7&6&6&5&5&3&5&5&4&4&3&3&4\\
M3 & structured + Q4 & 6&5&7&7&6&6&6&4&6&7&7&7&5&5&6\\
M4 & adaptive + Q4 & 7&6&8&7&6&7&6&5&7&7&7&7&5&5&7\\
\bottomrule
\end{tabular}%
}
\end{table*}

The state-level rates plotted in Figure~\ref{fig:motivation} follow directly from these counts. Examples highlight distinct mechanism roles. In s4-2, where too many branches compete for the remaining budget, bandit adaptation reaches 9/10 compared with 3/10 for the fixed reference, while staged evaluation reaches 8/10 by reducing evaluation waste rather than reallocating branch budget. In s3-2, the search structure itself restricts progress: meta-evolution reaches 6/10, recombination also reaches 6/10, and repair reaches only 2/10. The hardest execution-failure situation is s1-2, construction timeout, where even environment probing and repair drop to 7/10 and 6/10 because locating the bottleneck does not by itself perform the required algorithmic restructuring. These differences are why the main text reports state averages but the appendix retains situation-level outcomes.

\subsection{Joint Query and Operator Composition}
\label{app:motivation-joint}

We additionally examine query/operator composition across all three s2 situations. The reference uses self-only queries and local revision. Q-only replaces self-only with environment probing; O-only replaces local revision with repair; and Q+O applies both changes. Evaluation, memory, and strategy retain their reference settings. Each configuration therefore contributes 30 trials (ten per s2 situation). Average iterations measure the iteration of first success, conditioned on the trial succeeding.

\begin{figure}[ht]
    \centering
    \includegraphics[width=0.85\linewidth]{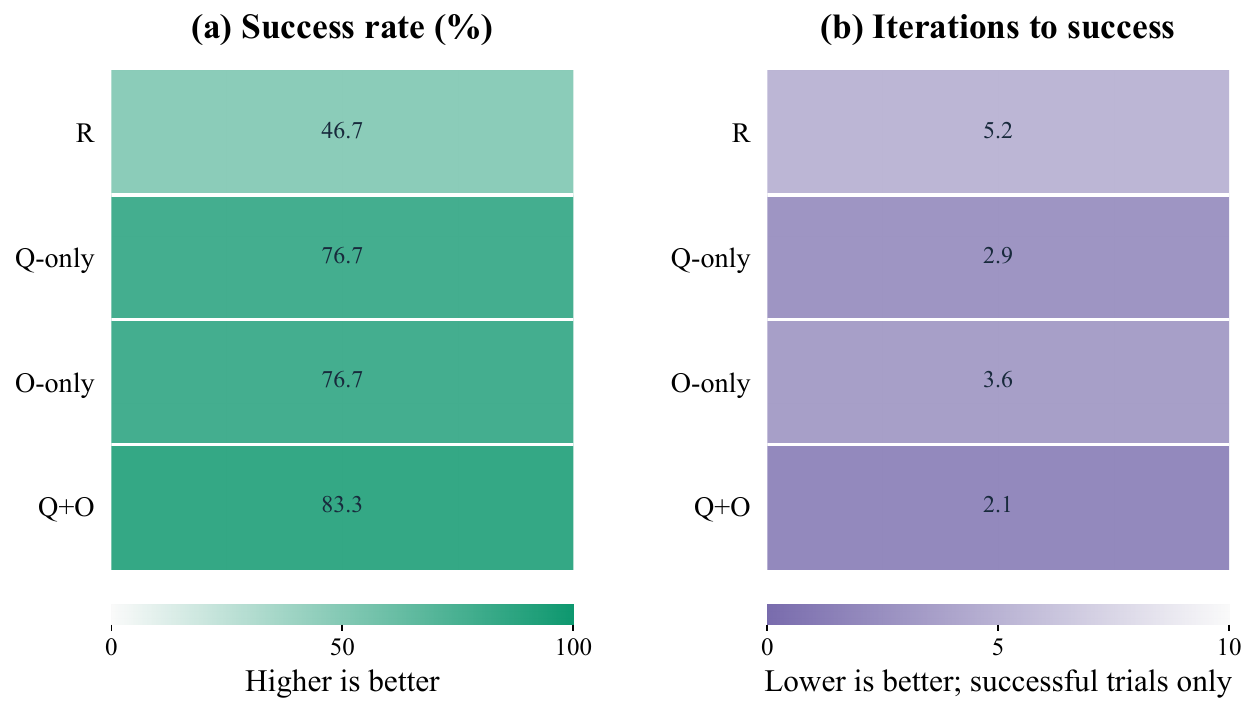}
    \caption{Separate and joint query/operator changes across the three s2 situations. \textbf{(a)} Success rate over 30 trials per configuration; higher is better. \textbf{(b)} Mean iterations to first success among successful trials only; lower is better. The reference, Q-only, O-only, and Q+O settings obtain 46.7\%, 76.7\%, 76.7\%, and 83.3\% success, respectively, with conditional mean iterations of 5.2, 2.9, 3.6, and 2.1.}
    \label{fig:motivation-joint}
\end{figure}

\begin{table}[ht]
    \centering
    \small
    \setlength{\tabcolsep}{5pt}
    \caption{Joint Q/O results across the three s2 situations. Average iterations are conditional on success.}
    \label{tab:motivation-joint}
    \begin{tabular}{@{}lllcc@{}}
        \toprule
        \textbf{Configuration} & \textbf{Query} & \textbf{Operator} & \textbf{Successes} & \textbf{Avg. iterations} \\
        \midrule
        R & self-only & local-revision & 14/30 & 5.2 \\
        Q-only & env-probe & local-revision & 23/30 & 2.9 \\
        O-only & self-only & repair & 23/30 & 3.6 \\
        Q+O & env-probe & repair & 25/30 & 2.1 \\
        \bottomrule
    \end{tabular}
\end{table}

The joint intervention improves the observed state-resolution rate by approximately 6.7 percentage points over either single change ($(25-23)/30\times100$) and shortens successful trajectories relative to both. This result is consistent with complementary composition on these constructed s2 situations, but it does not establish general superadditivity: the experiment covers one pair of dimensions on one task, and iteration counts need not correspond to equal token or wall-clock cost.

\section{Proofs of Theoretical Results}
\label{app:theory}

\subsection{Budgeted Process and Notation}
\label{app:theory-setup}

We analyze one task at a time, with initial history $h_0$, budget $B$, and a finite decision cap $H$. A history $h_t$ contains the task, candidate records, observations, remaining resources, and the internal state needed to specify continuation policies. A policy $\pi$ selects an action from a finite nonempty feasible set $\mathcal{A}(h_t)$, possibly through an observation summary $\phi(h_t)$. The history representation avoids assuming that the summary is Markov or fully informative.

\paragraph{Shared transitions and costs.}
Actions follow a common transition kernel $K(dh'\mid h,a)$. If $\kappa_t\geq0$ is the cost incurred at decision $t$, then
\begin{equation}
    b_0=B,\qquad b_{t+1}=b_t-\kappa_t\geq0,
    \qquad \sum_{t<T}\kappa_t\leq B.
    \label{eq:app-budget}
\end{equation}
The model includes the costs of decision formation as well as query, generation, evaluation, memory maintenance, and strategy adaptation. If two decision procedures incur different control costs, that difference must be represented in their execution modes or explicit control operations. It cannot be hidden inside a policy-dependent transition for an otherwise identical action. A structured Action Package is therefore the implementation's main decision record, while the theoretical history and action description also account for the computation needed to produce and execute it.

We use a scalar budget for notation. Multiple resource limits can instead be carried in the state with componentwise feasibility checks; the proofs rely on a common feasible process and terminal boundary, not a particular cost unit. Random execution costs are handled by enforced limits and specified interruption outcomes. No operation is treated as free merely because it fails.

\paragraph{Stopping and terminal utility.}
Stop is a feasible action from every active history and enters an absorbing state. Let $T\leq H$ be the number of decisions up to termination; when termination is explicit, the stopping action is included among $A_0,\ldots,A_{T-1}$. After absorption, the only action is a zero-cost no-op and the retained result remains unchanged. We pad trajectories to $H$ for the proof.

The terminal utility $F(h_H)$ is measured by a fixed task-specific rule in $[0,1]$, including a failure value in the same interval. If no valid candidate exists, the rule returns this failure value; otherwise it scores the best retained candidate meeting the final requirements. Optimization-time evaluation may use different feedback formats or stages, but does not redefine $F$. Boundedness is a property of the specified terminal utility, not a consequence of dividing an arbitrary benchmark score by a reference target. For a utility bounded in a known interval $[L,U]$ with $U>L$, a fixed positive affine normalization gives the stated form; the corresponding raw-unit loss bound is scaled by $U-L$. No such bounded interval is asserted for every raw metric in our experiments, so the numerical bound is not applied to the reported GPU, systems, or reference-relative scores. The qualitative success rates in the motivation study are not estimates of the continuation values below.

\paragraph{Reference continuation and action values.}
Fix a reference optimizer $\pi_0$ with a defined continuation from every relevant history. It may be randomized and internally adaptive. A stateful reference requires a specified way to reconstruct or maintain its internal state on histories reached by $\pi$. Define
\begin{align}
    J_B(\pi)&=\mathbb{E}_{\pi}[F(h_H)\mid h_0],\label{eq:app-policy-value}\\
    V^{\pi_0}(h_t)&=\mathbb{E}_{\pi_0}[F(h_H)\mid h_t],\label{eq:app-continuation-value}\\
    q(h_t,a)&=\int V^{\pi_0}(h')\,K(dh'\mid h_t,a).\label{eq:app-action-value}
\end{align}
The terminal boundary and common initial history imply
\begin{equation}
    V^{\pi_0}(h_H)=F(h_H),\qquad
    V^{\pi_0}(h_0)=J_B(\pi_0).
    \label{eq:app-value-boundaries}
\end{equation}
Since these quantities measure terminal quality, no additional immediate improvement reward is added to Equation~\eqref{eq:app-action-value}. Action cost instead affects continuation through the remaining budget in $h'$.

\paragraph{Opportunities and losses.}
For the feasible choices exposed to $\pi$, write
\begin{align}
    q^\star(h)&=\max_{a\in\mathcal{A}(h)}q(h,a),\\
    \Delta(h)&=q^\star(h)-V^{\pi_0}(h),\\
    \ell(h)&=q^\star(h)-\mathbb{E}_{A\sim\pi(\cdot\mid h)}[q(h,A)].
\end{align}
The notation $\pi(\cdot\mid h)$ includes policies implemented through $\phi(h)$. Selection loss is nonnegative. Opportunity is nonnegative when the reference choices are retained with identical behavior and cost:
\begin{equation}
    V^{\pi_0}(h)
    =\mathbb{E}_{a\sim\pi_0(\cdot\mid h)}q(h,a)
    \leq\max_{a\in\mathcal{A}(h)}q(h,a).
    \label{eq:app-reference-inclusion}
\end{equation}
The decomposition itself does not require $\Delta(h)\geq0$. Different action catalogs change the split into $\Delta$ and $\ell$; their difference remains the expected advantage of the selected action over the reference continuation.

\subsection{Proof of the Cumulative Opportunity--Loss Decomposition}
\label{app:global-improvement-proof}

\begin{proof}[Proof of Theorem~\ref{thm:global-adaptive-value}]
We first identify the contribution of a single decision and then accumulate these contributions through a sequence of hybrid policies.

\paragraph{Conditional decision advantage.}
Let $g(h)$ be the expected advantage of the action selected by $\pi$ over the reference continuation. Adding and subtracting the best feasible action value gives
\begin{align}
    g(h)
    &:=\mathbb{E}_{A\sim\pi(\cdot\mid h)}[q(h,A)]-V^{\pi_0}(h)\nonumber\\
    &=q^\star(h)-V^{\pi_0}(h)
      -\left(q^\star(h)-\mathbb{E}_{A\sim\pi(\cdot\mid h)}[q(h,A)]\right)\nonumber\\
    &=\Delta(h)-\ell(h).
    \label{eq:app-conditional-advantage}
\end{align}
This identity holds regardless of whether $g(h)$ is positive.

\paragraph{One decision replacement.}
For $m=0,\ldots,H$, define $\pi^{(m)}$ to use $\pi$ for its first $m$ decisions and $\pi_0$ thereafter. Thus, $\pi^{(0)}=\pi_0$ and $\pi^{(H)}=\pi$. The adjacent policies $\pi^{(m)}$ and $\pi^{(m+1)}$ share the same prefix distribution $d_m^\pi$ of $h_m$. Conditional on that history, the former uses the reference continuation and the latter executes one $\pi$ action before returning to it. Consequently,
\begin{align}
    J_B(\pi^{(m)})&=\mathbb{E}_{h_m\sim d_m^\pi}[V^{\pi_0}(h_m)],\\
    J_B(\pi^{(m+1)})&=\mathbb{E}_{h_m\sim d_m^\pi}
       \left[\mathbb{E}_{A_m\sim\pi(\cdot\mid h_m)}q(h_m,A_m)\right].
\end{align}
Subtracting yields
\begin{align}
    J_B(\pi^{(m+1)})-J_B(\pi^{(m)})
    &=\mathbb{E}_{h_m\sim d_m^\pi}
      \left[\mathbb{E}[q(h_m,A_m)\mid h_m]-V^{\pi_0}(h_m)\right]\nonumber\\
    &=\mathbb{E}_{h_m\sim d_m^\pi}[g(h_m)]\nonumber\\
    &=\mathbb{E}_{\pi}[\Delta_m-\ell_m].
    \label{eq:app-hybrid-difference}
\end{align}
All changes to later histories are already included in the continuation values; the argument does not couple future trajectories to be identical.

\paragraph{Accumulation over the run.}
Summing the adjacent differences telescopes:
\begin{align}
    J_B(\pi)-J_B(\pi_0)
    &=J_B(\pi^{(H)})-J_B(\pi^{(0)})\nonumber\\
    &=\sum_{m=0}^{H-1}\left(J_B(\pi^{(m+1)})-J_B(\pi^{(m)})\right)\nonumber\\
    &=\sum_{m=0}^{H-1}\mathbb{E}_{\pi}[\Delta_m-\ell_m]\nonumber\\
    &=\mathbb{E}_{\pi}\sum_{m=0}^{H-1}(\Delta_m-\ell_m)\nonumber\\
    &=\mathbb{E}_{\pi}\sum_{m<T}(\Delta_m-\ell_m).
    \label{eq:app-global-decomposition-proof}
\end{align}
The finite horizon and bounded values justify exchanging summation and expectation. The last equality uses the zero contributions after absorption.
\end{proof}

\paragraph{Equivalent value-function telescoping.}
The same result follows directly from conditional expectation:
\begin{align}
    \mathbb{E}_{\pi}\sum_{t=0}^{H-1}(\Delta_t-\ell_t)
    &=\sum_{t=0}^{H-1}\mathbb{E}_{\pi}
       \left[\mathbb{E}_{\pi}[V^{\pi_0}(h_{t+1})\mid h_t]-V^{\pi_0}(h_t)\right]\nonumber\\
    &=\mathbb{E}_{\pi}\sum_{t=0}^{H-1}
       \left(V^{\pi_0}(h_{t+1})-V^{\pi_0}(h_t)\right)\nonumber\\
    &=\mathbb{E}_{\pi}[V^{\pi_0}(h_H)]-V^{\pi_0}(h_0)\nonumber\\
    &=J_B(\pi)-J_B(\pi_0).
\end{align}
This is the finite-horizon, terminal-utility form of a performance-difference identity~\cite{schulman2015trust}.

\paragraph{Why the stopping action matters.}
If $\pi$ stops at an active history $h$, then
\begin{equation}
    q(h,\mathrm{Stop})-V^{\pi_0}(h)=F(h)-V^{\pi_0}(h).
\end{equation}
This accounts for any reference continuation value forgone by stopping. Omitting that decision would omit a potentially negative contribution. After absorption, both continuation values equal the retained quality and the subsequent terms are zero.

\subsection{Selection-Loss Bound and Global Improvement}
\label{app:selection-guarantee-proof}
\label{app:local-improvement-proof}

The capability assumption in Equation~\eqref{eq:near-optimal-selection} is inspired by analyses that state explicit generation and comparison conditions for LLM selection~\cite{chen2025provable}. That work analyzes answer correctness and particular selection algorithms. Here the assumed probability concerns near-optimal action continuation value under the current budget; it is not supplied by that result. The parameters may vary with history, and no independence across optimization decisions is required.

\paragraph{Bounding the conditional selection loss.}
Fix $h_t$ and let $X_t=q^\star(h_t)-q(h_t,A_t)$. Bounded terminal utility implies $0\leq X_t\leq1$. Let $\mathcal{G}_t=\{X_t\leq\epsilon_t\}$ and $p_t=\Pr(\mathcal{G}_t^c\mid h_t)\leq\delta_t$. Splitting the conditional expectation gives
\begin{align}
    \ell_t
    &=\mathbb{E}[X_t\mid h_t]\nonumber\\
    &=\mathbb{E}[X_t\mathbf{1}_{\mathcal{G}_t}\mid h_t]
      +\mathbb{E}[X_t\mathbf{1}_{\mathcal{G}_t^c}\mid h_t]\nonumber\\
    &\leq\epsilon_t\Pr(\mathcal{G}_t\mid h_t)+\Pr(\mathcal{G}_t^c\mid h_t)\nonumber\\
    &=(1-p_t)\epsilon_t+p_t\nonumber\\
    &=\epsilon_t+p_t(1-\epsilon_t)\nonumber\\
    &\leq\epsilon_t+\delta_t(1-\epsilon_t)\nonumber\\
    &=(1-\delta_t)\epsilon_t+\delta_t.
    \label{eq:app-local-loss-bound}
\end{align}
The penultimate step uses $0\leq\epsilon_t\leq1$. Substituting into the conditional advantage gives the single-decision bound:
\begin{align}
    \mathbb{E}[q(h_t,A_t)\mid h_t]-V^{\pi_0}(h_t)
    &=\Delta_t-\ell_t\nonumber\\
    &\geq\Delta_t-(1-\delta_t)\epsilon_t-\delta_t.
    \label{eq:local-improvement-bound}
\end{align}

\begin{proof}[Proof of Theorem~\ref{thm:global-improvement}]
Apply Equation~\eqref{eq:app-local-loss-bound} conditionally at every active history reached by $\pi$. Combining it with Theorem~\ref{thm:global-adaptive-value} yields
\begin{align}
    J_B(\pi)-J_B(\pi_0)
    &=\mathbb{E}_{\pi}\sum_{t<T}(\Delta_t-\ell_t)\nonumber\\
    &\geq\mathbb{E}_{\pi}\sum_{t<T}
       \left(\Delta_t-(1-\delta_t)\epsilon_t-\delta_t\right)\nonumber\\
    &=\mathbb{E}_{\pi}\sum_{t<T}\Delta_t
      -\mathbb{E}_{\pi}\sum_{t<T}
       \left((1-\delta_t)\epsilon_t+\delta_t\right).
    \label{eq:app-global-improvement-bound}
\end{align}
This is Equation~\eqref{eq:global-improvement-bound}.
\end{proof}

\paragraph{Interpretation and scope.}
The bound yields the sufficient condition
\begin{equation}
\begin{gathered}
    \mathbb{E}_{\pi}\sum_{t<T}\Delta_t
    >
    \mathbb{E}_{\pi}\sum_{t<T}
    \bigl((1-\delta_t)\epsilon_t+\delta_t\bigr)\\
    \Longrightarrow\quad J_B(\pi)>J_B(\pi_0).
\end{gathered}
\label{eq:strict-global-improvement}
\end{equation}
It does not establish that the premise holds for every adaptive policy. The exact decomposition also permits negative local advantages. Writing $g_t^+=\max\{g(h_t),0\}$ and $g_t^-=\max\{-g(h_t),0\}$ gives
\begin{equation}
    J_B(\pi)-J_B(\pi_0)
    =\mathbb{E}_{\pi}\sum_{t<T}g_t^+
     -\mathbb{E}_{\pi}\sum_{t<T}g_t^-.
\end{equation}
Thus, positive contributions may compensate for unfavorable decisions. The result compares expected terminal quality at a shared budget, not monotonic separation of two realized trajectories or strict superiority in the infinite-budget limit. We do not establish the conditional selection premise for our Controller or infer it from endpoint scores. Finite-catalog expansion holds existing action values and costs fixed; catalog-dependent prompt overhead or changed continuation policies fall outside that particular comparison.

\subsection{Theoretical Remarks and Applicability}
\label{app:theory-remarks}

\paragraph{Remark on Theorem 1.}
The contribution of a decision is $\Delta_t-\ell_t$, so a larger $\Delta_t$ is useful only to the extent that selection realizes the additional value. At a fixed history, expanding the action set while preserving its existing actions and costs cannot decrease $q^\star(h_t)$. However, if the selected-action distribution remains unchanged, the increase in $q^\star(h_t)$ raises $\Delta_t$ and $\ell_t$ equally, leaving their difference unchanged. Thus, component diversity must be paired with selection that actually exploits the added choices. The summation also permits negative $\Delta_t-\ell_t$ at some steps: such decisions are compatible with overall improvement when positive contributions elsewhere outweigh them.

\paragraph{Remark on Theorem 2.}
Writing the bound on $\ell_t$ as $\epsilon_t+\delta_t(1-\epsilon_t)$ reveals that, for fixed $\delta_t$, making $\epsilon_t$ small still leaves a bound close to $\delta_t$. Refining already near-optimal choices is therefore insufficient to tighten the guarantee when failures to select them remain frequent. This motivates using diagnostics and accumulated experience to address recurring poor selections, as well as improving the quality of successful choices. Such evidence is not free: additional querying or reflection consumes budget and changes $q(h_t,a)$ itself. The analysis therefore motivates allocating decision effort according to the state, rather than increasing it uniformly. A positive lower bound certifies improvement; a nonpositive one leaves the comparison unresolved.

\paragraph{Theoretical perspective and implementation scope.}
Theorem~\ref{thm:global-adaptive-value} provides an exact performance-difference identity for the shared finite-horizon process, while Theorem~\ref{thm:global-improvement} establishes a general performance envelope under the conditional selection assumption (Equation~\ref{eq:near-optimal-selection}). In practice, our LLM-based Controller and Strategy Adapter operate directly on observed textual and metric feedback without explicitly computing or estimating abstract analytical quantities such as $q$, $\Delta_t$, or $\ell_t$. Furthermore, while the theoretical analysis assumes a normalized bounded utility to cleanly isolate cumulative credit, empirical evaluations involve task-specific, unnormalized score scales. This formulation positions the theory as a foundational design guide rather than an online runtime estimator, offering rigorous first-principles motivation for why state-conditioned mechanism coordination is essential.

\section{Framework Specification}
\label{app:implementation}

This appendix specifies the architecture in Section~\ref{sec:framework}: the information exposed to each component, the decisions that may change, and the execution conditions that remain fixed. Task-specific parameter values, prompts used in a particular run, and experimental implementation alignment belong to the experimental record.

\subsection{OptiCom Algorithm and Execution Order}
\label{app:algorithm}

Algorithm \ref{alg:optimize-anything} provides the pseudo-code logic driving the OptiCom framework loop.

\vspace{-3mm}
\begin{algorithm}[h!]
\caption{\ours{}: Adaptive Optimization through Component Composition}
\label{alg:optimize-anything}
\begin{algorithmic}[1]
\STATE \textbf{Input:} Task $d$, mechanisms $\mathcal{R}$, initial candidates $\mathcal{P}_0$, budget $B$.
\STATE $\mathcal{P} \gets \mathcal{P}_0$; $\mathcal{M} \gets \emptyset$; $g \gets g_0$; $b \gets B$
\WHILE{$\neg\operatorname{Stop}(\mathcal{P},\mathcal{M},b)$}
    \STATE \COMMENT{Step 1: Select a composition from the current state}
    \STATE $z \gets \operatorname{Summarize}(d,\mathcal{P},\mathcal{M},b)$
    \STATE $(a,c_{\mathrm{ctrl}}) \gets \operatorname{Controller}(z,\mathcal{R},g,b)$
    \STATE $b \gets b-c_{\mathrm{ctrl}}$
    \IF{$\neg\operatorname{Executable}(a,b)$}
        \STATE \textbf{break}
    \ENDIF

    \STATE \COMMENT{Step 2: Query, generate, and evaluate within the remaining budget}
    \STATE $(\mathcal{X},r,c_{\mathrm{exec}})
        \gets \operatorname{ExecuteWithHarness}(a,d,\mathcal{P},\mathcal{M},b)$
    \STATE $b \gets b-c_{\mathrm{exec}}$

    \STATE \COMMENT{Step 3: Retain candidates and experience using the selected policy}
    \STATE $(\mathcal{P},\mathcal{M})
        \gets \operatorname{Update}(\mathcal{P},\mathcal{M},a,\mathcal{X},r)$

    \STATE \COMMENT{Step 4: Revise strategy guidance from accumulated outcomes}
    \IF{$\operatorname{AdaptNow}(\mathcal{M},b)$}
        \STATE $(g,c_{\mathrm{adapt}})
            \gets \operatorname{Adapter}(g,\mathcal{M},b)$
        \STATE $b \gets b-c_{\mathrm{adapt}}$
    \ENDIF
\ENDWHILE
\STATE \textbf{return} $\operatorname{BestValid}(\mathcal{P})$
    \COMMENT{Return task-defined failure if no valid candidate exists}
\end{algorithmic}
\end{algorithm}
\vspace{-3mm}

\subsection{Component Interfaces and Action Packages}
\label{app:component-interfaces}

\paragraph{Task contract and registry.}
The task adapter defines artifact parsing and serialization, the editable region, the optimization objective, mandatory feasibility and correctness checks, permitted evidence sources, the final scoring rule, and resource limits. A registry entry identifies a component's accepted inputs, outputs, applicability conditions, parameter schema, and execution limits. Null or minimal entries permit configurations that do not use a particular function. A new mechanism can be registered when it satisfies a contract; requirements beyond the contract must be exposed as extensions rather than hidden inside an opaque component.

\paragraph{State records.}
A candidate record contains an identifier, serialized artifact or artifact reference, parent identifiers, the producing action and strategy version, evaluation status, and links to evaluation records. Evaluation records distinguish validity, objective measurements, diagnostics, and costs. Missing or incomplete checks remain explicitly unknown. Experience entries link an observation or lesson to the candidates, actions, and evidence from which it was derived. These records support reconstruction of the visible context and prevent an unverified proposal from being confused with an evaluated incumbent.

\begin{table}[ht]
\centering
\small
\setlength{\tabcolsep}{4pt}
\renewcommand{\arraystretch}{1.12}
\caption{Action Package fields and their execution responsibilities. Fields select registered functionality; they do not redefine the task contract.}
\label{tab:action-package}
\begin{tabularx}{\linewidth}{@{}l>{\raggedright\arraybackslash}p{0.22\linewidth}>{\raggedright\arraybackslash}X@{}}
\toprule
\textbf{Field} & \textbf{Decision} & \textbf{Consumer and constraints}\\
\midrule
$q_t$ & Query plan & Query executor: permitted tools, evidence requests, and retrieval operations; may be empty.\\
$c_t$ & Context scope & Context builder: candidate identifiers, diagnostics, and experience to expose within the context limit.\\
$o_t$ & Update operator & Operator executor: a registered mechanism with compatible artifact inputs and required parent candidates.\\
$w_t$ & Branch width & Scheduler: a positive bounded number of candidate proposals, subject to available resources.\\
$e_t$ & Evaluation plan & Evaluator: registered stages, effort limits, and feedback form; final required checks remain unchanged.\\
$\rho_t$ & Retention policy & Archive and experience stores: which evaluated candidates, records, and lessons to retain or expose.\\
\bottomrule
\end{tabularx}
\end{table}

\paragraph{Functional correspondence.}
The task adapter supplies $A$; query and context fields invoke $Q$; the selected update supplies $O$; and the evaluation plan invokes $E$. Archive and experience operations implement $M$. The Controller and Adapter jointly implement $S$. Branch width and evaluation effort are resource-allocation decisions coordinated by $S$, rather than additional AQOEMS dimensions. A retained memory entry belongs to $M$, while the decision to retrieve it for a particular update belongs to $Q$.

\paragraph{Records versus theoretical quantities.}
The runtime records observed scores, feedback, and resource use. It does not require a learned value function, an explicit estimate of $q(h,a)$, or a numerical estimate of $\Delta_t$ and $\ell_t$. Theoretical histories can retain more information than the compact state summary sent to the Controller.

\subsection{Controller Context and Action Validation}
\label{app:controller-details}

\paragraph{Decision context.}
The Controller receives the task contract, current incumbent and selected alternatives, recent diagnostics, a summary of progress and resource consumption, relevant experience, applicable registry entries, and current strategy preferences. Its instruction specifies the Action Package schema and asks it to choose a feasible next decision. The policy may preserve a working configuration; adaptation does not require changing every dimension at every iteration. Context limits and retrieval rules determine what information is actually visible.

\paragraph{Validation sequence.}
The runtime first parses the structured output and checks required fields and types. It then checks that component identifiers exist, input artifacts and parent references are available, and any operator preconditions are satisfied. For example, recombination requires compatible parents; a tool query requires task permission and an available tool. Finally, it checks branch, context, evaluation, and execution limits against the updated budget after charging the Controller call.

A bounded correction policy may request a repaired package or select a declared feasible fallback. Every additional model call is charged, and the configured correction limit is finite. If no valid affordable action remains, the run terminates with the best verified incumbent or the task-defined failure outcome. No malformed package authorizes arbitrary tool access or an unbounded retry loop.

\paragraph{Execution order.}
Let $u_t$ be the context assembled from the selected query and context scope. One iteration follows
\begin{align}
    u_t&=\operatorname{QueryContext}(d,z_t,q_t,c_t),\\
    \mathcal{Y}_t&=\operatorname{Propose}(d,u_t,o_t,w_t),\\
    \mathcal{R}_t&=\operatorname{HarnessEvaluate}(d,\mathcal{Y}_t,e_t),\\
    (\mathcal{P}_{t+1},\mathcal{D}_{t+1},\mathcal{H}_{t+1})
      &=\operatorname{Retain}(\mathcal{P}_t,\mathcal{D}_t,\mathcal{H}_t,
                              \mathcal{Y}_t,\mathcal{R}_t,\rho_t).
\end{align}
Each execution stage observes its resource limit and reports actual usage. The initial query is downstream of package selection: its new evidence can guide proposal generation and subsequent decisions but is not retroactively part of the Controller's input. A separate replanning call, if explicitly configured, is another charged decision and must be recorded as such.

\subsection{Strategy Adaptation Rules}
\label{app:adapter-details}

\paragraph{Adaptable state.}
The strategy state $\theta$ contains selection preferences, weights over registered operators, and templates used to form component instructions. A strategy version identifies the resolved choices used by each Action Package. The Adapter reads accumulated records, including successful and failed attempts, validity outcomes, observed progress, and consumed resources. It proposes changes to these preferences and templates while retaining the task and component contracts.

\paragraph{Trigger and update.}
The trigger is a configured predicate over the outcome history and elapsed iterations since the last adaptation. It may combine a stagnation window or a recurring-failure condition with a minimum adaptation interval. These windows and thresholds are task or run parameters, not universal constants. The scheduler checks that adaptation is affordable before invoking it. A proposal must preserve the schema, use registered identifiers, and satisfy any configured ranges; invalid proposals leave the previous strategy active. Accepted changes are recorded with their supporting outcomes and version.

Changes to preferences or templates are policy updates, not evidence of improvement by themselves. An Adapter update can be ineffective or harmful, so its value must be evaluated through subsequent behavior. In particular, the architecture does not assume monotonically decreasing selection error or cost-free adaptation.

\paragraph{Scope of modification.}
The Adapter cannot silently change the final objective, required tests, artifact permissions, or resource accounting. A newly introduced executable mechanism requires registration and contract validation. This keeps strategy revision distinct from unrestricted changes to runtime logic and makes its effects traceable through the action records.

\subsection{Archive, Memory, and Context Construction}
\label{app:memory-context}

\paragraph{Candidate archive.}
The archive stores artifacts with their lineage and evaluation records. Retention policies may preserve high-quality candidates, diverse alternatives, or candidates relevant to unresolved failures. The best candidate that has completed all required checks remains available for return. Candidates with only intermediate evaluation are explicitly marked as provisional, and a higher provisional score alone does not displace a verified incumbent.

\paragraph{Experience memory.}
Experience entries describe observations such as a repeated failure mode, a useful local modification, or a combination that did not justify its cost. Entries retain their source evidence and scope instead of converting one outcome into an unconditional rule. A configuration with no additional experience memory still retains the current candidate and the minimal records required to execute and account for the loop.

\paragraph{Retrieval and context.}
The query policy chooses whether to use the current candidate only, inspect additional diagnostics, or retrieve relevant archive and experience entries. The context builder applies the selected scope and size limit, keeping candidate identities and evaluation provenance available. Retrieval can improve the information presented to the Controller in a later state or to the current operator after package selection; those two information paths are recorded separately.

\subsection{Execution Reliability and Resource Accounting}
\label{app:execution-reliability}

\paragraph{Evaluation and final eligibility.}
The evaluation protocol may return scalar scores, textual diagnostics, or staged feedback. Staging can reject invalid candidates before expensive measurements, but omitted checks remain incomplete rather than passed. Final eligibility is determined by the task contract, and the returned candidate must satisfy its mandatory checks. Changing feedback richness does not change the final scoring rule.

\paragraph{Failure handling.}
Execution records distinguish malformed packages, invalid artifact edits, parse or build failures, failed checks, timeouts, and budget interruptions. Artifact updates are applied to a candidate copy so that a failed edit does not destroy the retained incumbent. Task-appropriate process isolation and time limits bound candidate execution. A failure returns a typed record with the available diagnostics and actual costs, which can inform later state summaries.

\paragraph{Resource ledger.}
For each decision, the ledger records control, query, generation, evaluation, memory-processing, and adaptation usage, including failures and bounded retries. Token counts, elapsed execution time, and API charges remain distinguishable rather than being treated as interchangeable. A run declares its controlling budget or conversion rule and checks remaining resources between stages. Any final verification or adaptation must fit within the declared allocation. A finite iteration cap additionally prevents unlimited zero-cost decision cycles.

\paragraph{Architecture and experimental alignment.}
The contracts in this appendix specify the intended executable architecture. A reported experimental configuration must identify its implemented Controller, Adapter, prompts, registry, task adapter, and accounting behavior. A rule-based controller or a descriptive evaluation-depth field that does not affect evaluator execution is a different implementation choice and must be reported as such. Appendix~\ref{app:alignment-checks} distinguishes functional mapping criteria from evidence of implementation equivalence; the criteria are not themselves a completed equivalence audit.

\section{Experimental Protocols}
\label{app:experimental-protocols}

\subsection{Benchmarks and Evaluation Pipelines}
\label{app:benchmark-details}

We integrate the benchmark evaluation pipelines from SkyDiscover~\cite{liu2026skydiscover} into our framework. The integration adapts candidate submission and the return of evaluation records. The study uses the task definitions and available checks from these pipelines, with the final score fields and aggregation rules specified below. All compared methods therefore use the same task-specific evaluation criteria. Instantiating an optimizer within the shared configuration space changes its optimization procedure, not the benchmark objective.

The evaluation covers 32 groups: 17 mathematical tasks, five systems tasks, three GPU Mode tasks, and one group each for Frontier-CS, ARC, Sky Festival, HotpotQA, QNN, ALE-Bench, and KernelBench. Table~\ref{tab:benchmark-overview} summarizes these domains. The main results table presents seven representative benchmarks; aggregate rankings use all 32 groups in Table~\ref{tab:complete-best-results}. MLA Decode is excluded.

\begin{table}[ht]
    \centering
    \caption{Benchmark domains and evaluation targets. Evaluation uses task-specific benchmark harnesses and the stated aggregation rules.}
    \label{tab:benchmark-overview}
    \begingroup
    \small
    \setlength{\tabcolsep}{4pt}
    \renewcommand{\arraystretch}{1.12}
    \begin{tabular}{
        @{}
        p{0.12\textwidth}
        p{0.24\textwidth}
        p{\dimexpr0.64\textwidth-16pt\relax}
        @{}
    }
        \toprule
        \textbf{Domain} & \textbf{Benchmark family} & \textbf{Evaluation target} \\
        \midrule
        Math
        & Mathematical optimization
        & Numerical objectives subject to task-specific feasibility constraints. \\
        Systems
        & ADRS
        & Workload performance, cost, and composite system objectives. \\
        GPU
        & GPU Mode; KernelBench
        & Numerical correctness and kernel execution performance; scaled inverse runtime or eager-baseline speedup, respectively. \\
        Algorithms
        & Frontier-CS; ALE-Bench-Lite
        & Bounded algorithmic-problem scores or private contest-performance scores over the stated problem sets. \\
        Reasoning
        & ARC
        & Correctness of candidate transformations on benchmark inputs. \\
        Creative
        & Sky Festival
        & Satisfaction of semantic and compositional image requirements. \\
        Prompts
        & HotpotQA
        & Question-answering performance obtained with optimized instructions. \\
        Quantum
        & QNN circuit topology
        & Classification performance obtained with candidate circuit structures. \\
        \bottomrule
    \end{tabular}
    \endgroup
\end{table}

Where a benchmark separates optimization feedback from final evaluation, this separation is retained. Within-evaluator aggregation is distinguished from the aggregation across independent optimization subtasks. For example, runtime aggregation within a GPU benchmark remains part of that evaluator and is distinct from averaging scores across independent optimization subtasks.

\paragraph{Score fields and aggregation.}
GPU Mode (VecAdd, Grayscale, and TriMul) reports $3000/t_g$, where $t_g$ is the geometric mean execution time in microseconds after the required correctness tests; these are scaled inverse runtimes, not speedup ratios. KernelBench instead reports speedup relative to PyTorch eager. Frontier-CS averages the bounded scores over all 172 problems and divides by 100, counting failed or missing problems as zero. ARC uses final held-out test pass@2, not optimization-time cell accuracy. ALE-Bench-Lite averages private \texttt{final\_performance} over ten problems without dividing by 100. HotpotQA uses exact-match percentage divided by 100. QNN uses held-out classification accuracy with 60 test examples per evaluation; five-run summaries are computed before rounding. Sky Festival averages no physical execution metric: its score is the sum of seven vision-model rubric categories divided by 100.

For PRISM, one run is scored by the arithmetic mean of the per-configuration \texttt{score} outputs across the full configuration set. Table~\ref{tab:complete-best-results} reports the maximum of these run-level means over five runs. This study's aggregation is not the inverse of an averaged KV-cache pressure plus a success-rate term. PRISM scores are ranked in the higher-is-better direction used in the result matrix. Mathematical tasks use their task-specific reference-relative scores, which can exceed one; no cross-task mean of raw scores is reported.

\subsection{Models, Baselines, and Budgets}
\label{app:model-budget-protocol}

The main experiments use Doubao-Seed-2.0-pro as the optimization backbone. Each method is evaluated over five independent runs with a budget of 30 optimization iterations per task. Model parameters remain fixed throughout optimization. Task-specific evaluation models, where applicable, are shared across the compared methods.

The 13 method-inspired profiles in the overall comparison are TextGrad, OPRO, ProTeGi, Reflexion, GEPA, AdaEvolve, EvoX, FunSearch, AlphaEvolve, Voyager, DSPy, Self-Refine, and LATS. The main table presents five of these profiles plus the official SkyDiscover EvoX and AdaEvolve implementations. The overall ranking includes only the 13 profiles and \ours{}; official-code entries are not additional overall methods. Their mappings into the shared configuration space and implementation differences are described in Appendix~\ref{app:optimizer-mapping}.

The shared iteration budget controls the number of optimization rounds. It does not imply identical numbers of generated candidates, LLM calls, or tokens. Accordingly, the main comparison concerns final quality under a common iteration limit; the iteration-wise trajectories do not establish equal-time or equal-cost performance.

\subsection{Aggregation and Reporting}
\label{app:statistical-protocol}

\paragraph{Benchmark-level scores.}
Let $y_{m,t,r}$ denote the final task score of method $m$ on subtask $t$ in run $r$. For a benchmark group $g$ containing a fixed collection of independent subtasks $\mathcal{T}_g$, we compute
\begin{equation}
    G_{m,g,r}
    =
    \frac{1}{|\mathcal{T}_g|}
    \sum_{t\in\mathcal{T}_g} y_{m,t,r}.
    \label{eq:benchmark-group-score}
\end{equation}
The average uses the comparable scores specified for that group, rather than heterogeneous physical quantities. Single-task benchmarks correspond to $|\mathcal{T}_g|=1$. Test cases evaluated jointly by a task's evaluator are not counted as separate optimization subtasks.

\paragraph{Mean and maximum scores.}
For the higher-is-better scores reported in the main table, the two summaries are
\begin{equation}
    \operatorname{Mean}_{m,g}
    =
    \frac{1}{5}\sum_{r=1}^{5}G_{m,g,r},
    \qquad
    \operatorname{Max}_{m,g}
    =
    \max_{r\in\{1,\ldots,5\}}G_{m,g,r}.
    \label{eq:mean-max-reporting}
\end{equation}
Both use the same five runs. For a group containing multiple subtasks, group aggregation precedes taking the maximum over runs. Thus, Max does not combine a separately selected best run from each subtask.

\paragraph{Overall ranking.}
Figure~\ref{fig:overall_results}(a) reports two rankings computed from the same 14 configurations and 32 groups. The Max ranking uses $\operatorname{Max}_{m,g}$ and the Mean ranking uses $\operatorname{Mean}_{m,g}$. For each statistic, configurations are ranked within each group in descending score order; ties at the retained reporting precision receive the arithmetic mean of their occupied rank positions. A method's overall rank is the arithmetic mean of its 32 within-group ranks, with every group receiving equal weight regardless of its number of subtasks. The two official-code entries are excluded. Numeric zero scores remain in the ranking and are not treated as missing. Under this convention, \ours{} has average ranks 1.72 (Max) and 2.08 (Mean). The Max matrix contains all 448 profile entries; the corresponding aggregate Mean matrix is used for the Mean-based ranking.

\paragraph{Reporting standards.}
Overall ranks are computed from the retained four-decimal aggregate score matrices following standard multi-task algorithmic evaluation protocols. Rounding can create ties, especially on Cloudcast and EPLB. The main table reports five-run Max and Mean; ablation tables additionally retain the supplied standard-deviation summaries. Best-of-five performance reflects peak search capability under a fixed budget, serving as a standard evaluation metric across complex heuristic search spaces.

\paragraph{Relative gains.}
For each column of Table~\ref{tab:main-results}, the relative gain over the strongest compared baseline is
\begin{equation}
    \operatorname{Gain}
    =
    \frac{s_{\mathrm{OptiCom}}-s_{\mathrm{baseline}}}
         {s_{\mathrm{baseline}}}
    \times 100\%.
    \label{eq:relative-improvement}
\end{equation}
The baseline is selected separately for the Mean and Max columns. These values describe relative score improvements, not percentage-point changes or statistical significance.

\subsection{Ablation Configurations}
\label{app:ablation-protocol}

The core ablations use Doubao-Seed-2.0-pro on Heilbronn Triangle, LLM-SQL, and HotpotQA. Each configuration is evaluated over five independent runs with a maximum of 30 optimization iterations. We report the mean and standard deviation of the final retained scores, rather than statistics over intermediate iterations. Task definitions and final evaluation criteria remain unchanged. The iteration cap does not impose an equal token, API-call, or monetary budget across configurations.

\paragraph{Core interventions.}
The variants modify the following parts of the optimizer:
\begin{itemize}
    \item \textbf{Fixed composition.}
    The optimizer retains a predefined configuration of query, operator, evaluation, and memory mechanisms throughout the run, with the Strategy Adapter disabled. Candidates, evaluation records, and memory contents continue to update under these fixed rules. Thus, a fixed composition does not imply repeatedly generating the same candidate or ignoring new feedback.

    \item \textbf{Random composition.}
    Mechanisms are selected randomly from executable combinations in the same registry, without state-conditioned selection preferences or Strategy Adapter updates. Applicability checks and mandatory task constraints remain in force.

    \item \textbf{Without Strategy Adapter.}
    The Controller continues to select Action Packages from the current optimization state, but the initial strategy guidance, selection weights, and templates remain fixed. Immediate action selection is therefore adaptive even though its longer-term guidance is not revised.

    \item \textbf{Operator-only adaptation.}
    The Strategy Adapter is disabled, and the Controller can change only the update operator. Query, evaluation, memory, and other action settings follow fixed rules. Comparing this variant with the preceding one tests the value of allowing a broader range of optimization decisions to change.

    \item \textbf{Without Experience Memory.}
    Cross-iteration experience summaries and their retrieval are removed. The candidate archive, current evaluation feedback, and progress and budget records remain available, so candidate retention and operations requiring archived candidates remain executable. The Controller and Strategy Adapter otherwise remain enabled.
\end{itemize}

These interventions distinguish component contents from the rules governing their use. For example, freezing a memory mechanism does not prevent it from storing new observations, while removing experience memory does not remove the candidate archive. Comparisons involving the Adapter assess its contribution within the iteration-limited process, including its additional API calls.

\paragraph{Adaptation frequency.}
On Heilbronn Triangle, we compare the default event-triggered Adapter with no adaptation, periodic adaptation every five iterations, and adaptation after every iteration. No Adapter call is made after the final iteration because its output would have no subsequent optimization step to influence. The observed mean call counts are therefore $0$, $10$, $5$, and $29$, respectively. We additionally report the mean total API token consumption for each complete run, including optimization-related calls beyond candidate generation.

\paragraph{Feedback access.}
The feedback experiment on Heilbronn Triangle crosses two information conditions with the presence or absence of the Strategy Adapter. Basic feedback exposes validity status and a scalar score; rich feedback additionally exposes failure reasons and task diagnostics. The underlying evaluator and final scoring criterion remain unchanged. The intended intervention concerns information available to the optimizer, rather than changes to what constitutes a valid or high-quality solution. Results are reported with total API token consumption because richer feedback can also increase the amount of processed context.

\paragraph{Initialization sensitivity.}
We evaluate three deliberately different initial Q/O/E/M configurations on Heilbronn Triangle. The local-revision configuration uses self-only context, local candidate updates, scalar evaluation feedback, and a quality-oriented archive with recent records. The diagnosis-and-repair configuration uses environment probing, targeted repair or bottleneck-directed updates, diagnostic feedback, and structured experience. The exploration-and-recombination configuration retrieves alternative candidates, combines their structures or construction procedures, and maintains a quality-and-diversity-oriented archive. All configurations retain the same mandatory validity checks and final objective.

For each initialization, the fixed variant retains its mechanism rules throughout optimization. The adaptive variant executes the specified configuration in the first iteration and subsequently permits the Controller and Strategy Adapter to modify it. Within each comparison, the initial candidate set and evaluation protocol are held constant. These three fixed configurations are distinct from the default Fixed composition variant in Table~\ref{tab:core-ablations}; they provide additional, deliberately designed reference configurations rather than repeated measurements of that row.

\section{Additional Experimental Results}
\label{app:full-results}

\subsection{Complete Max-Score Results}
\label{app:main-results}

Table~\ref{tab:complete-best-results} contains the retained Max scores for 32 groups and 14 configurations, totaling 448 entries. It replaces the earlier partial 22-task archive. The corresponding aggregate Mean matrix is used for the Mean-based ranking in Figure~\ref{fig:overall_results}(a) but is not reproduced here to avoid duplicating another full 32-by-14 table. Unqualified baseline names denote method-inspired profiles; the two official-code comparisons appear only in Table~\ref{tab:main-results}. Values are reported to four decimal places and follow the task-specific higher-is-better score definitions. They should be compared within a task, not averaged across heterogeneous raw scales.

This matrix is sufficient to reproduce the reported Max-based ranks at the retained precision. It is not a replacement for the underlying five-run logs or candidate artifacts. Ties share the highest/second-highest distinct-score annotations; ranking instead uses average occupied positions as specified in Appendix~\ref{app:statistical-protocol}.

\begin{sidewaystable}[p]
    \centering
    \caption{Max scores over five runs on all 32 benchmark groups. Higher is better within each task. Baseline columns are method-inspired framework profiles. Bold and underlining mark the highest and second-highest distinct values at four-decimal precision. The final row gives average within-task ranks with average ranks for ties; lower is better. Official-code entries are excluded.}
    \label{tab:complete-best-results}
    \begingroup
    \small
    \setlength{\tabcolsep}{3pt}
    \renewcommand{\arraystretch}{1.13}
    \resizebox{\linewidth}{!}{%
    \begin{tabular}{@{}l*{14}{r}@{}}
        \toprule

Task & TextGrad & OPRO & ProTeGi & Reflexion & GEPA & AdaEvolve & EvoX & FunSearch & AlphaEvolve & Voyager & DSPy & Self-Refine & LATS & \ours{} \\
\midrule
Circle packing & 0.9256 & 0.8892 & 0.8806 & 0.7652 & 0.8607 & 0.8259 & 0.8833 & 0.8277 & \textbf{1.0003} & 0.6818 & 0.7830 & 0.7547 & 0.9282 & \underline{0.9682} \\
Circle packing (rect.) & 0.9869 & 0.9895 & 0.9968 & 0.9964 & 0.9896 & 0.9876 & 0.8577 & 0.9895 & 0.9913 & 0.9967 & 0.9879 & 0.7342 & \underline{0.9975} & \textbf{0.9982} \\
Erdos minimum overlap & 0.7618 & 0.7920 & 0.7633 & 0.7637 & 0.7618 & 0.7618 & 0.7600 & 0.7530 & \underline{0.8204} & 0.7695 & 0.7637 & 0.0000 & 0.8044 & \textbf{0.9972} \\
Autocorrelation inequality 1 & 0.9939 & 0.9950 & 0.9945 & 0.9928 & 0.9920 & 0.9945 & \textbf{0.9971} & 0.9948 & 0.9957 & 0.9925 & 0.9949 & 0.9934 & 0.9960 & \underline{0.9966} \\
Heilbronn convex (13) & 0.6415 & 0.7213 & 0.4706 & 0.7141 & 0.5436 & 0.4591 & \underline{0.8242} & 0.0878 & 0.0425 & 0.4211 & 0.7136 & 0.5160 & 0.7147 & \textbf{0.9386} \\
Heilbronn convex (14) & \underline{0.7384} & 0.5486 & 0.6956 & 0.6117 & 0.6052 & 0.3543 & 0.7236 & 0.6689 & 0.5780 & 0.4616 & 0.6672 & 0.6592 & 0.4531 & \textbf{0.8591} \\
Heilbronn triangle & 0.8557 & 0.5984 & 0.8282 & 0.6401 & 0.7191 & 0.7685 & \underline{0.9431} & 0.6508 & 0.7993 & 0.6060 & 0.9145 & 0.7151 & 0.8874 & \textbf{0.9608} \\
Hexagon packing (11) & \underline{0.0000} & \underline{0.0000} & \underline{0.0000} & \underline{0.0000} & \underline{0.0000} & \underline{0.0000} & \underline{0.0000} & \underline{0.0000} & \underline{0.0000} & \underline{0.0000} & \underline{0.0000} & \underline{0.0000} & \underline{0.0000} & \textbf{0.2399} \\
Hexagon packing (12) & 0.7884 & \textbf{0.8942} & 0.6706 & 0.6838 & 0.6362 & 0.7167 & \underline{0.8871} & 0.6555 & 0.6570 & 0.6362 & 0.6570 & 0.7167 & 0.8731 & 0.7642 \\
Matrix multiplication & \underline{0.8000} & \underline{0.8000} & \underline{0.8000} & \underline{0.8000} & \underline{0.8000} & \underline{0.8000} & \underline{0.8000} & \underline{0.8000} & \textbf{0.8421} & 0.0000 & 0.0000 & \underline{0.8000} & 0.0000 & \textbf{0.8421} \\
Min-max distance (2D) & 0.8675 & 0.7159 & 0.8475 & \underline{0.9989} & 0.8559 & 0.7167 & 0.7161 & 0.8718 & 0.8741 & 0.7161 & 0.9494 & 0.8625 & 0.9942 & \textbf{0.9996} \\
Min-max distance (3D) & \underline{0.9415} & 0.7307 & 0.8803 & 0.9083 & 0.8803 & 0.8994 & 0.8803 & 0.9170 & 0.8803 & 0.6943 & 0.6943 & 0.9031 & 0.8778 & \textbf{0.9968} \\
Autocorrelation inequality 2 & 0.9549 & 0.9827 & \underline{0.9879} & 0.8344 & 0.9606 & 0.9864 & \textbf{1.0135} & 0.8347 & 0.9833 & 0.9549 & 0.9574 & 0.9794 & 0.9466 & 0.9872 \\
Signal processing & 0.5012 & 0.5298 & 0.5354 & 0.5124 & 0.5442 & \textbf{0.6086} & 0.5558 & 0.5183 & 0.5232 & 0.5209 & \underline{0.5914} & 0.5343 & 0.5147 & 0.5484 \\
Sums/differences of finite sets & 0.8880 & 0.0000 & 0.9384 & 0.9497 & 0.9295 & 0.9066 & 0.0000 & 0.8667 & \textbf{0.9531} & 0.9188 & 0.0000 & 0.9242 & 0.8800 & \underline{0.9525} \\
Autocorrelation inequality 3 & \textbf{0.9952} & 0.9913 & 0.9941 & 0.9940 & \underline{0.9951} & 0.9912 & 0.9939 & 0.9919 & 0.9920 & 0.9934 & 0.9921 & 0.9924 & 0.9938 & \textbf{0.9952} \\
Uncertainty inequality & 0.8956 & 0.8958 & 0.9120 & 0.8804 & 0.8995 & 0.9034 & \underline{0.9133} & 0.8933 & 0.9090 & 0.8804 & \underline{0.9133} & 0.8960 & 0.9016 & \textbf{0.9195} \\
Cloudcast & \underline{0.0010} & \underline{0.0010} & \underline{0.0010} & \underline{0.0010} & \underline{0.0010} & \underline{0.0010} & \underline{0.0010} & \underline{0.0010} & \underline{0.0010} & \underline{0.0010} & \underline{0.0010} & \underline{0.0010} & \underline{0.0010} & \textbf{0.0011} \\
EPLB & 0.1274 & 0.1274 & 0.1274 & \underline{0.1275} & \underline{0.1275} & \underline{0.1275} & 0.1274 & \underline{0.1275} & 0.1274 & 0.1266 & \underline{0.1275} & \underline{0.1275} & 0.1274 & \textbf{0.1344} \\
LLM-SQL & 0.6769 & 0.6488 & 0.6299 & 0.6385 & 0.6589 & 0.6190 & 0.6835 & 0.5829 & 0.5876 & 0.6451 & 0.1146 & 0.6554 & \underline{0.6863} & \textbf{0.6980} \\
PRISM & 0.8630 & 0.9584 & \underline{0.9619} & 0.0389 & 0.9603 & 0.9613 & 0.1734 & 0.9614 & 0.9589 & 0.9614 & \textbf{0.9620} & 0.9589 & \textbf{0.9620} & 0.9612 \\
Transaction scheduling & 3875.9690 & 3984.0637 & 3759.3985 & 3937.0079 & 3906.2500 & \underline{4016.0643} & 3663.0037 & 3690.0369 & 3952.5692 & 3690.0369 & 3690.0369 & 3690.0369 & 3690.0369 & \textbf{4032.2581} \\
Grayscale & 42.1582 & 58.7491 & 65.3024 & 38.9915 & 51.4823 & \textbf{69.1047} & 47.6258 & 55.8310 & 62.4172 & 36.5749 & 44.8931 & 53.2185 & 60.7564 & \underline{68.3296} \\
TriMul & 2.3921 & 2.3773 & 2.5410 & 2.4127 & 2.4370 & 2.5730 & \underline{2.5930} & 2.3875 & 2.5912 & 2.5691 & 2.5131 & 2.4918 & 2.5760 & \textbf{2.6540} \\
VecAdd & 89.4521 & 120.3184 & 95.7632 & 112.4891 & 86.1298 & \underline{131.0543} & 127.3412 & 98.6504 & 115.8237 & 109.4310 & 91.2785 & 124.7659 & 85.9934 & \textbf{133.2014} \\
HotpotQA & 0.3452 & 0.3128 & \underline{0.4410} & 0.3891 & 0.3810 & 0.3970 & 0.3860 & 0.3605 & 0.3274 & 0.3719 & 0.3386 & 0.3540 & 0.4130 & \textbf{0.5130} \\
ARC & 0.3142 & 0.3875 & \underline{0.4730} & 0.3421 & 0.3960 & 0.4130 & 0.4610 & 0.3956 & 0.3208 & 0.3764 & 0.3559 & 0.3093 & 0.4570 & \textbf{0.5030} \\
Frontier-CS & 0.5381 & 0.6124 & \underline{0.6320} & 0.5709 & 0.6130 & 0.6210 & 0.6260 & 0.5042 & 0.5933 & 0.5487 & 0.5215 & 0.6078 & 0.6270 & \textbf{0.6930} \\
Sky Festival & 0.2451 & \underline{0.3102} & 0.1987 & 0.2845 & 0.2134 & \textbf{0.3291} & 0.2678 & 0.2950 & 0.1842 & 0.2319 & 0.2765 & 0.3014 & 0.2056 & 0.2983 \\
KernelBench & 0.6234 & 0.8412 & 1.0539 & 0.5127 & 0.9845 & 0.7321 & \underline{1.0892} & 0.6750 & 0.9104 & 0.5899 & 0.8876 & 0.7954 & 0.6481 & \textbf{1.1123} \\
ALE-Bench & 1842.1593 & 1950.4821 & 1715.3094 & 1888.7512 & 1921.6045 & 1795.8236 & 1913.1962 & \underline{1973.9548} & 1742.0671 & 1811.3925 & 1905.7482 & 1768.5139 & 1837.9264 & \textbf{1991.2057} \\
QNN & 0.7833 & 0.8167 & \underline{0.8833} & 0.7667 & 0.8000 & 0.8333 & 0.8167 & 0.8333 & 0.8000 & 0.8333 & 0.7667 & 0.8333 & 0.8500 & \textbf{0.9000} \\
\midrule
Average rank & 8.703125 & 8.015625 & 6.453125 & 8.625000 & 7.875000 & 7.000000 & 6.328125 & 8.875000 & 7.562500 & 10.078125 & 8.453125 & 8.468750 & 6.843750 & \textbf{1.718750} \\
\bottomrule
    \end{tabular}%
    }
    \endgroup
\end{sidewaystable}

\subsection{Optimization Trajectories}
\label{app:optimization-trajectories}

Figures~\ref{fig:app-progress-p1}--\ref{fig:app-progress-p6} complement the aggregate results with optimization trajectories on twelve mathematical benchmarks. The plots show best-so-far scores from individual archived runs of the framework profiles, rather than averages over five runs or trajectories of the official-code entries. Highlighted annotations indicate the iteration at which \ours{} reaches its final recorded best score. These trajectories reveal when improvements occur and whether early advantages persist. Iterations describe search progress rather than equal computational cost, and benchmark scores should be interpreted within each task; a score close to one does not, by itself, certify proximity to a mathematical optimum.

\paragraph{Early progress and final quality capture different properties.}
On Erdos Min Overlap, \ours{} establishes a substantial advantage within the first few iterations and subsequently makes smaller improvements, reaching a recorded score of $0.997$ at iteration $16$ (Figure~\ref{fig:app-progress-p1}a). A different pattern appears on Circle Packing Rect: several baselines obtain strong candidates before \ours{}, but \ours{} subsequently overtakes them and reaches $0.998$ at iteration $9$ (Figure~\ref{fig:app-progress-p2}a). Early progress can also fail to translate into the strongest final result. On Circle Packing, \ours{} reaches $0.968$ at iteration $3$, while AlphaEvolve later obtains a higher score (Figure~\ref{fig:app-progress-p5}a). These comparisons show why optimizer quality cannot be characterized solely by either the first few iterations or the final score: their relative importance depends on the available budget and the quality required by the task.

\paragraph{Temporary stagnation need not imply exhausted improvement opportunities.}
On Heilbronn Triangle, \ours{} remains near $0.774$ for several iterations before improving to $0.961$ at iteration $13$ (Figure~\ref{fig:app-progress-p3}a). On Matmul, it initially trails several methods at approximately $0.582$, then improves to $0.800$ and finally $0.842$ at iteration $11$, matching the strongest final result shown (Figure~\ref{fig:app-progress-p6}a). These trajectories illustrate that an optimizer can recover from an unproductive interval or an initially unfavorable position. They motivate retaining alternatives and reconsidering optimization decisions when progress stalls, rather than treating a short plateau as evidence that further search is unproductive. The score histories alone do not identify which mechanism produced each improvement, but they demonstrate the importance of evaluating decisions over the remaining optimization horizon.

\paragraph{The useful refinement horizon varies across tasks.}
Some trajectories are dominated by large early gains, whereas others continue to accumulate small improvements after reaching a strong candidate. For example, \ours{} reaches a high score early on both Minimizing Max Min Dist tasks, but its final recorded improvements occur at iterations $26$ and $13$, respectively (Figure~\ref{fig:app-progress-p4}). On Circle Packing Rect and First Autocorr Ineq, the full-scale curves appear nearly indistinguishable among several methods, while the insets expose meaningful differences in the timing and magnitude of subsequent refinements (Figure~\ref{fig:app-progress-p2}). Similarly, Sums Diffs Finite Sets shows incremental progress up to iteration $20$, with \ours{} finishing close to the strongest baseline (Figure~\ref{fig:app-progress-p5}b). These patterns motivate adapting the balance between exploration and refinement to recent progress and remaining budget. They also show why the iteration of the final improvement is insufficient as a standalone efficiency measure: a late, small refinement may follow a much earlier attainment of practically useful quality.

\paragraph{Adaptive composition does not dominate every task.}
The additional trajectories expose limitations alongside favorable results. EvoX exceeds \ours{} on both autocorrelation tasks (Figures~\ref{fig:app-progress-p1}b and~\ref{fig:app-progress-p2}b), and several methods attain higher final scores on Hexagon Packing 12 (Figure~\ref{fig:app-progress-p6}b). On the latter task, \ours{} improves to approximately $0.764$ but does not close the gap to the strongest alternatives. Thus, continued improvement within a run does not necessarily imply competitive final performance. These cases are consistent with the distinction in Section~\ref{sec:theory}: making complementary mechanisms available creates opportunities, but realizing their value also requires selecting suitable actions. The trajectories do not distinguish insufficient alternatives from inaccurate selection or inadequate feedback; resolving those explanations requires controlled comparisons or action-level records.

Taken together, the trajectories support a design centered on evolving optimization states: preserve strong candidates, retain the ability to make further improvements after stagnation, and adjust the search effort as the remaining opportunities change. This interpretation connects to the cumulative opportunity--loss decomposition in Section~\ref{sec:theory}, under which intermediate decisions matter through their contribution to the final outcome. The plots illustrate the resulting temporal patterns without directly estimating $\Delta_t$ or $\ell_t$, or attributing the observed gains to an individual component.

\begin{figure}[ht]
    \centering
    \includegraphics[width=\textwidth]{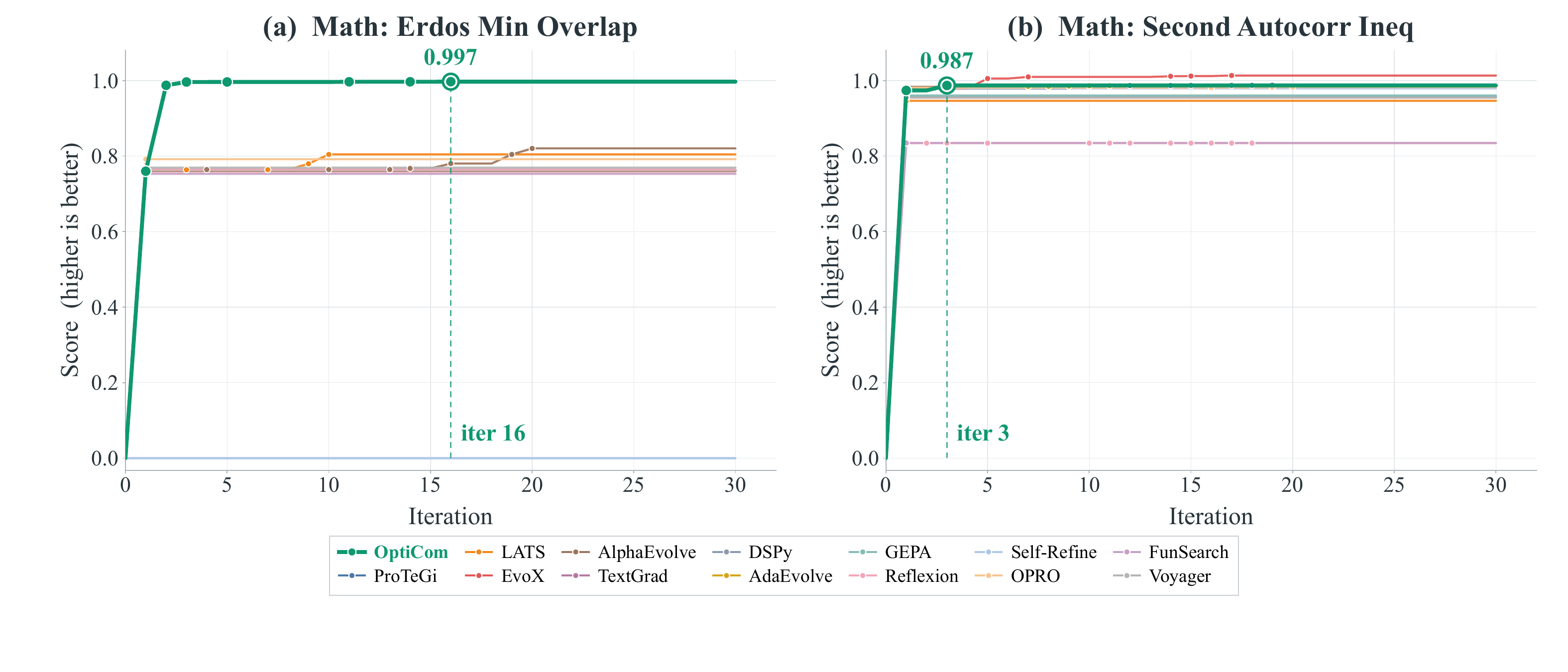}
    \caption{Archived best-so-far scores on \textbf{(a)} Erdos Min Overlap and \textbf{(b)} Second Autocorr Ineq. \ours{} establishes an early advantage on Erdos Min Overlap, whereas EvoX obtains a higher final score on Second Autocorr Ineq. Highlighted annotations mark the final recorded best score of \ours{} and its corresponding iteration.}
    \label{fig:app-progress-p1}
\end{figure}

\begin{figure}[ht]
    \centering
    \includegraphics[width=\textwidth]{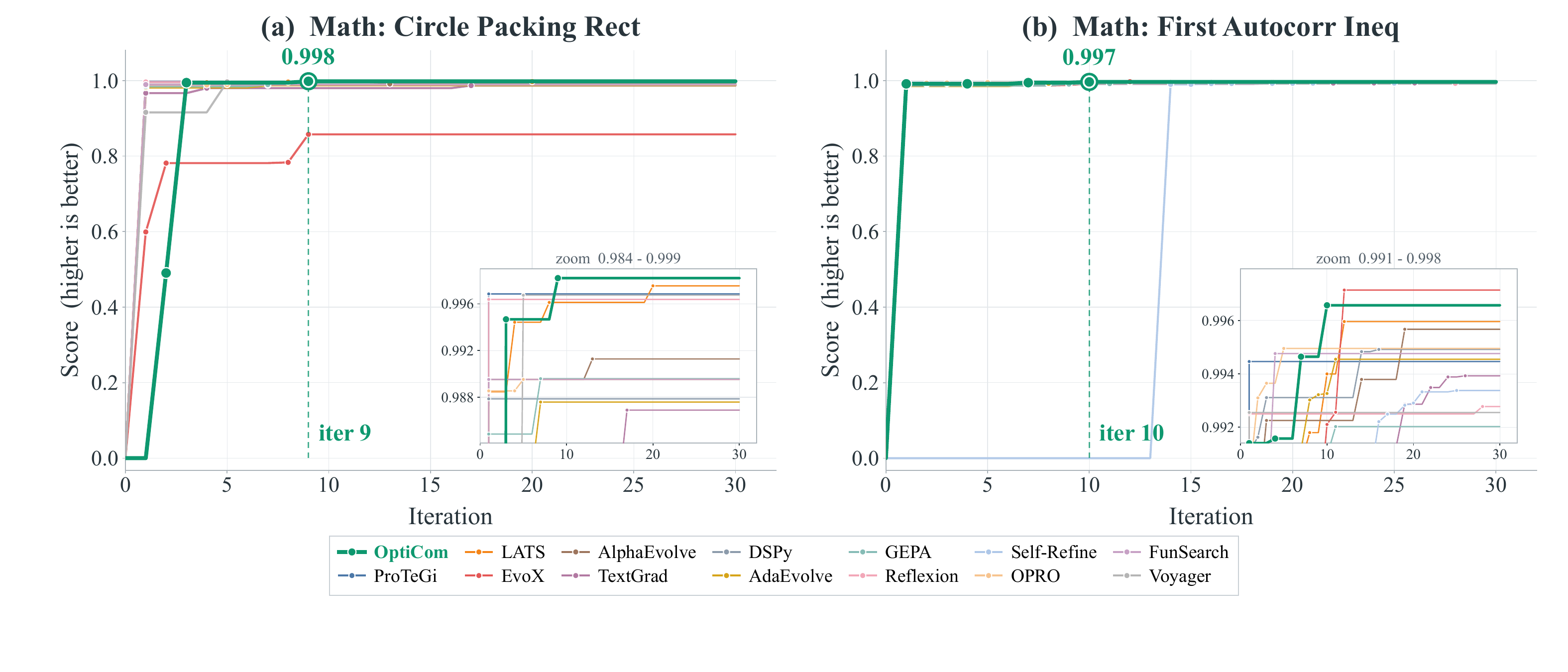}
    \caption{Archived best-so-far scores on \textbf{(a)} Circle Packing Rect and \textbf{(b)} First Autocorr Ineq. Insets reveal differences among high-scoring candidates that are difficult to distinguish on the full vertical scale. \ours{} achieves the highest final score on Circle Packing Rect, while EvoX finishes slightly higher on First Autocorr Ineq.}
    \label{fig:app-progress-p2}
\end{figure}

\begin{figure}[ht]
    \centering
    \includegraphics[width=\textwidth]{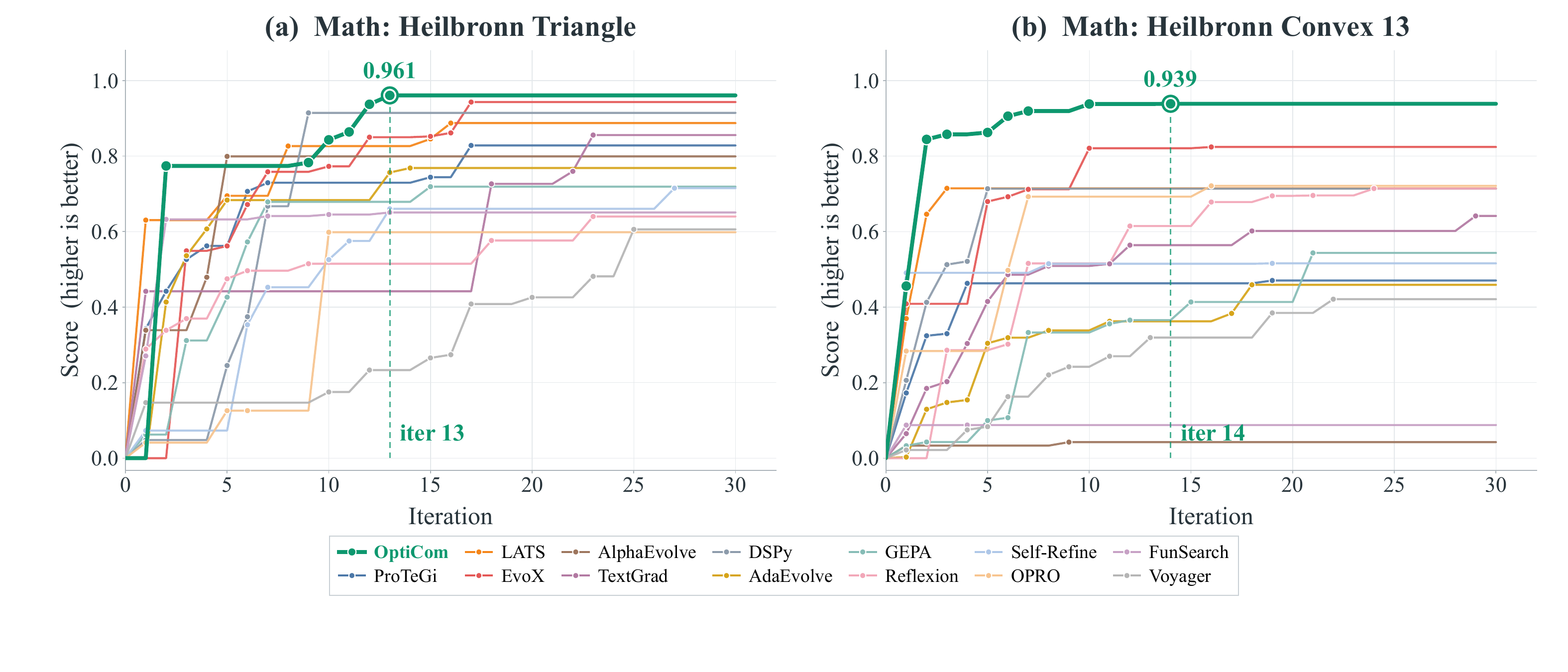}
    \caption{Archived best-so-far scores on \textbf{(a)} Heilbronn Triangle and \textbf{(b)} Heilbronn Convex 13. On Heilbronn Triangle, \ours{} resumes improvement after an initial plateau and reaches $0.961$ at iteration $13$. On Heilbronn Convex 13, successive improvements produce a final recorded score of $0.939$ at iteration $14$.}
    \label{fig:app-progress-p3}
\end{figure}

\begin{figure}[ht]
    \centering
    \includegraphics[width=\textwidth]{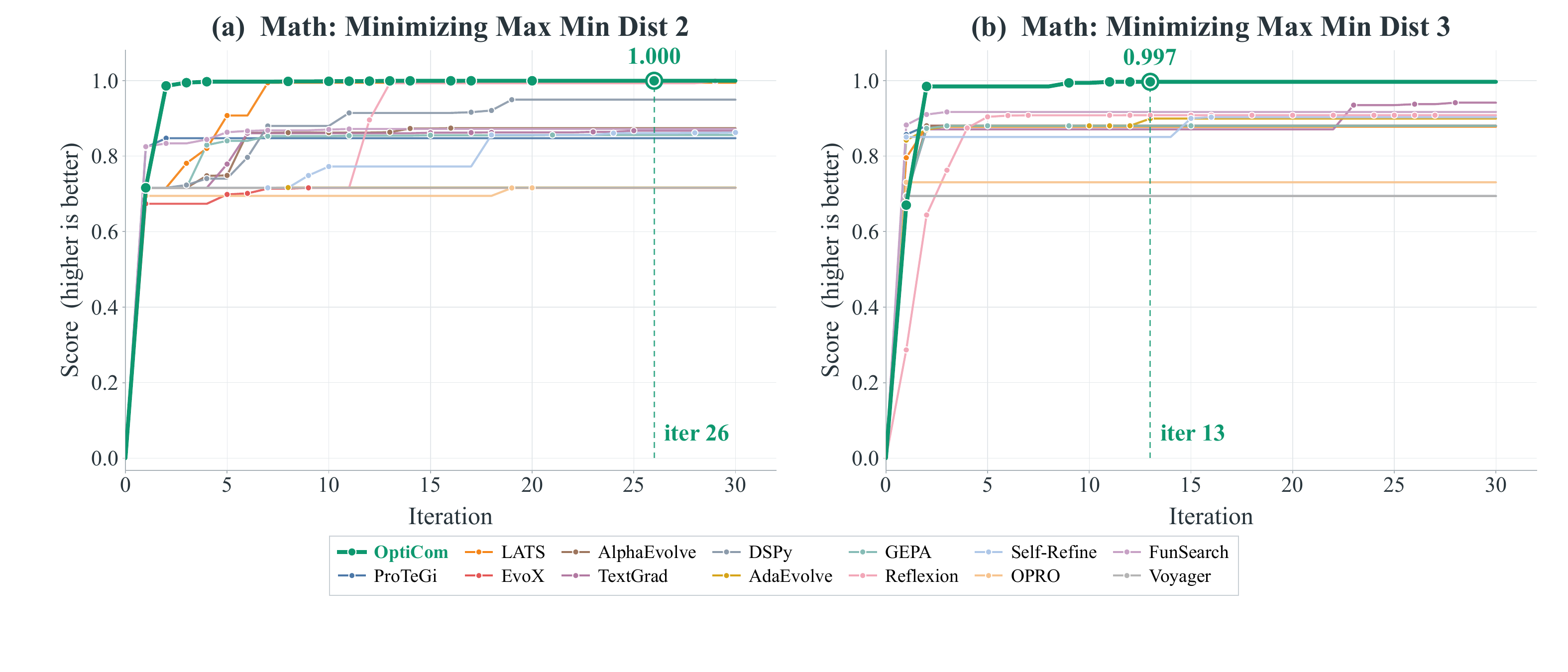}
    \caption{Archived best-so-far scores on \textbf{(a)} Minimizing Max Min Dist 2 and \textbf{(b)} Minimizing Max Min Dist 3. Large early improvements are followed by smaller refinements, with the final recorded improvements of \ours{} occurring at iterations $26$ and $13$, respectively. The displayed score of $1.000$ is rounded and does not constitute an optimality certificate.}
    \label{fig:app-progress-p4}
\end{figure}

\begin{figure}[ht]
    \centering
    \includegraphics[width=\textwidth]{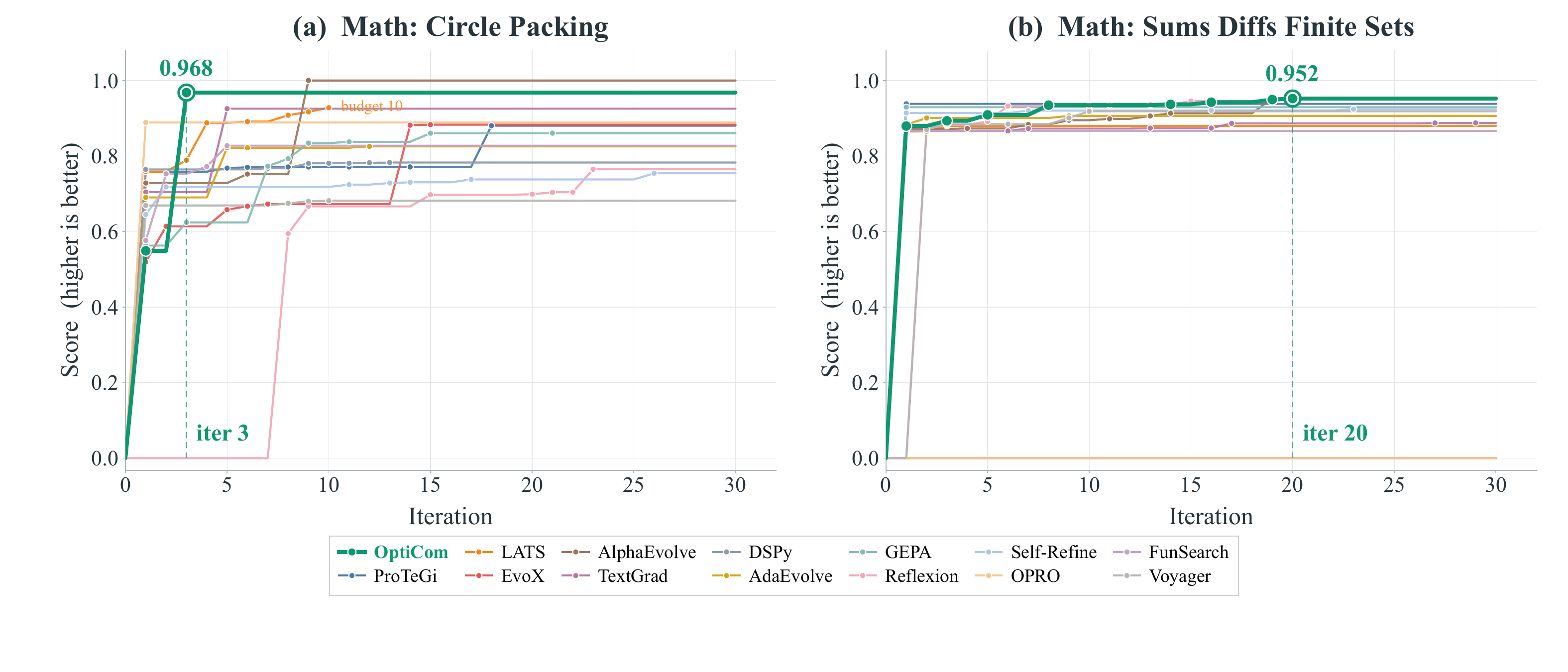}
    \caption{Archived best-so-far scores on \textbf{(a)} Circle Packing and \textbf{(b)} Sums Diffs Finite Sets. \ours{} makes a large early improvement on Circle Packing, but AlphaEvolve later achieves a higher score. On Sums Diffs Finite Sets, smaller improvements accumulate until iteration $20$. The Circle Packing panel marks the shorter recorded budget of LATS.}
    \label{fig:app-progress-p5}
\end{figure}

\begin{figure}[ht]
    \centering
    \includegraphics[width=\textwidth]{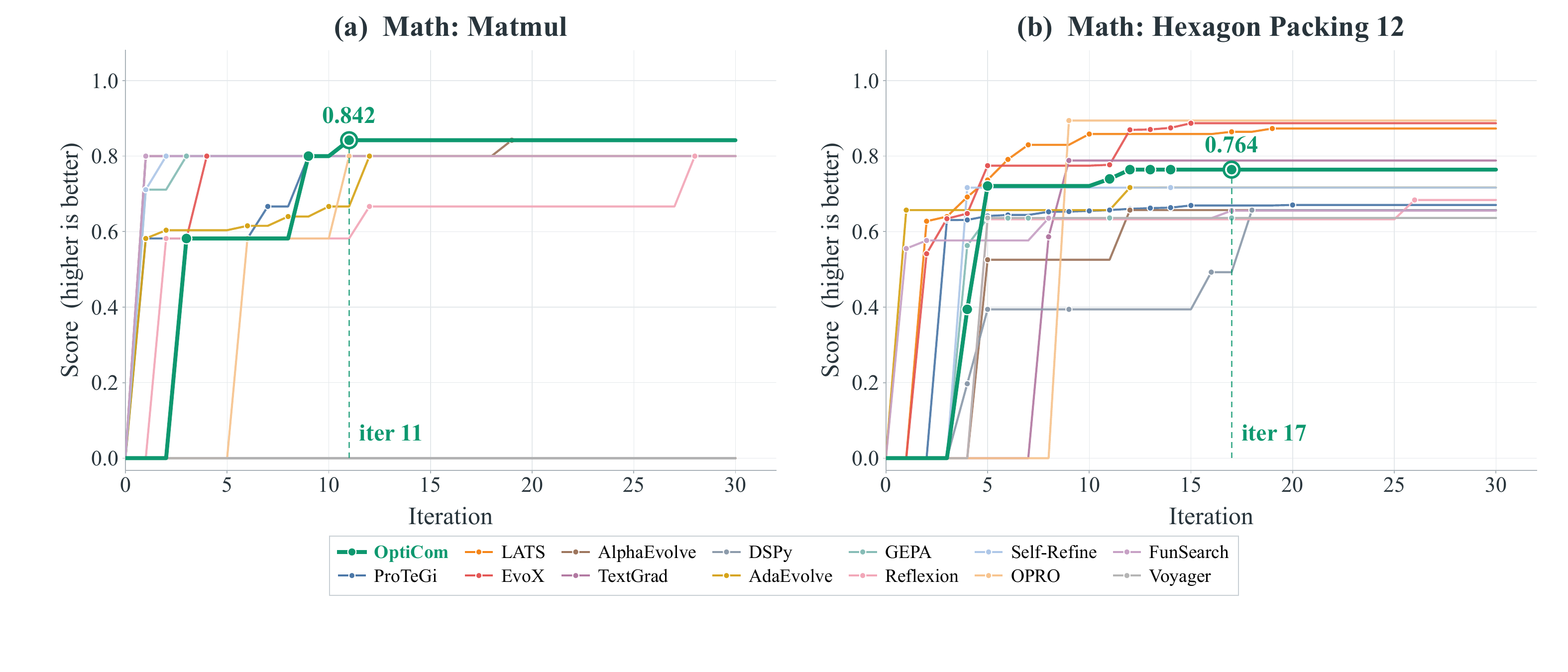}
    \caption{Archived best-so-far scores on \textbf{(a)} Matmul and \textbf{(b)} Hexagon Packing 12. On Matmul, \ours{} overcomes an initial plateau and reaches $0.842$ at iteration $11$, matching the strongest final score shown. On Hexagon Packing 12, improvements to $0.764$ remain insufficient to match several baselines, illustrating a limitation of the observed optimization trajectory.}
    \label{fig:app-progress-p6}
\end{figure}

\subsection{Cross-Backbone Robustness}
\label{app:transfer-results}

Table~\ref{tab:cross-backbone-full} reproduces the five-run cross-backbone summaries from Table~\ref{tab:cross-backbone}. All optimization-related LLM calls use the listed backbone, with the remaining configuration and evaluation procedures unchanged. The experiment therefore replaces the optimization backbone as a whole rather than the candidate generator alone.

To characterize the contribution of adaptation within each backbone, Table~\ref{tab:cross-backbone-gains} reports the difference between the full framework and the variant without the Adapter. The difference is positive in every backbone--task pair. Relative to the corresponding no-Adapter mean, the gains range from approximately $23.1\%$ to $33.7\%$ on Heilbronn Triangle and from $11.1\%$ to $18.1\%$ on HotpotQA.

\begin{table}[ht]
    \centering
    \small
    \setlength{\tabcolsep}{4pt}
    \renewcommand{\arraystretch}{1.12}
    \caption{Cross-backbone evaluation over five runs per configuration, with a maximum of 30 iterations. All optimization-related LLM calls use the listed backbone. Bold indicates the higher mean within each backbone--task pair.}
    \label{tab:cross-backbone-full}
    \begin{tabular}{@{}lcccc@{}}
        \toprule
        & \multicolumn{2}{c}{\textbf{Heilbronn Triangle}}
        & \multicolumn{2}{c}{\textbf{HotpotQA}} \\
        \cmidrule(lr){2-3}
        \cmidrule(lr){4-5}
        \textbf{Backbone} &
        \textbf{Full \ours{}} &
        \textbf{w/o Adapter} &
        \textbf{Full \ours{}} &
        \textbf{w/o Adapter} \\
        \midrule
        Doubao-Seed-2.0-pro &
        $\mathbf{0.953 \pm 0.008}$ &
        $0.713 \pm 0.027$ &
        $\mathbf{0.501 \pm 0.017}$ &
        $0.451 \pm 0.058$ \\
        GPT-5.5 &
        $\mathbf{0.987 \pm 0.002}$ &
        $0.796 \pm 0.011$ &
        $\mathbf{0.553 \pm 0.012}$ &
        $0.471 \pm 0.037$ \\
        Claude Opus 4.6 &
        $\mathbf{0.973 \pm 0.007}$ &
        $0.787 \pm 0.012$ &
        $\mathbf{0.554 \pm 0.010}$ &
        $0.469 \pm 0.041$ \\
        GLM-5.3 &
        $\mathbf{0.969 \pm 0.008}$ &
        $0.783 \pm 0.014$ &
        $\mathbf{0.523 \pm 0.019}$ &
        $0.461 \pm 0.059$ \\
        Kimi-K3 &
        $\mathbf{0.974 \pm 0.004}$ &
        $0.791 \pm 0.011$ &
        $\mathbf{0.547 \pm 0.016}$ &
        $0.465 \pm 0.053$ \\
        \bottomrule
    \end{tabular}
\end{table}

\begin{table}[ht]
    \centering
    \small
    \setlength{\tabcolsep}{8pt}
    \renewcommand{\arraystretch}{1.12}
    \caption{Absolute mean-score gains from enabling the Strategy Adapter, computed as Full \ours{} minus w/o Adapter using Table~\ref{tab:cross-backbone}. These are within-backbone comparisons under the same 30-iteration cap.}
    \label{tab:cross-backbone-gains}
    \begin{tabular}{@{}lcc@{}}
        \toprule
        \textbf{Backbone} &
        \textbf{Heilbronn Triangle} &
        \textbf{HotpotQA} \\
        \midrule
        Doubao-Seed-2.0-pro & $+0.240$ & $+0.050$ \\
        GPT-5.5 & $+0.191$ & $+0.082$ \\
        Claude Opus 4.6 & $+0.186$ & $+0.085$ \\
        GLM-5.3 & $+0.186$ & $+0.062$ \\
        Kimi-K3 & $+0.183$ & $+0.082$ \\
        \bottomrule
    \end{tabular}
\end{table}

The effect is not restricted to a backbone with a particularly weak no-Adapter result. On Heilbronn Triangle, the four alternative backbones already exceed the Doubao no-Adapter mean, yet each still benefits from enabling adaptation. Conversely, the size of the gain does not follow a common ordering across tasks: Doubao has the largest absolute gain on Heilbronn Triangle, whereas Claude Opus 4.6 has the largest on HotpotQA. The evidence therefore favors a benefit that persists across the tested backbones, rather than a monotonic relationship between backbone performance and the value of adaptation.

The full framework also has lower observed standard deviations in all ten comparisons. On HotpotQA, standard deviations range from $0.010$ to $0.019$ with the Adapter, compared with $0.037$ to $0.059$ without it. This pattern is consistent with more repeatable outcomes in these runs, although five-run summaries alone do not establish statistical significance or a universal stability guarantee.

Finally, this experiment directly evaluates the cross-backbone contribution of the Strategy Adapter, not every individual Q/O/E/M mechanism. Together with the core ablations, it supports the portability of the proposed design while leaving component-specific transfer effects unresolved. Since total computational expenditure is not matched across configurations or providers, the results should not be interpreted as a ranking of backbone cost efficiency.

\subsection{Additional Ablation Details}
\label{app:ablation-results}

\paragraph{How frequently should strategy adaptation occur?}
Table~\ref{tab:adapter-frequency} shows that all three adaptation schedules outperform the no-adaptation variant on Heilbronn Triangle. The default event-triggered policy obtains a mean score of $0.953$, compared with $0.954$ when adaptation is invoked after every nonterminal iteration. It uses approximately $65.5\%$ fewer Adapter calls, although the reduction in total token consumption is about $1.37\%$. Thus, the evidence supports obtaining similar observed mean quality without invoking the Adapter at every iteration, rather than a substantial reduction in total computation. Every-iteration adaptation has the smallest observed standard deviation, so the results do not establish that less frequent adaptation is uniformly preferable.

Periodic adaptation every five iterations obtains a lower mean score of $0.939$. However, it also makes fewer Adapter calls than the event-triggered variant. This comparison therefore combines differences in timing and frequency and does not isolate the advantage of state-dependent triggering at a matched number of updates.

\begin{table}[ht]
    \centering
    \small
    \setlength{\tabcolsep}{5pt}
    \renewcommand{\arraystretch}{1.12}
    \caption{Strategy adaptation frequency on Heilbronn Triangle. Scores report mean $\pm$ standard deviation over five runs. Adapter calls and total API tokens are run averages; K denotes one thousand tokens. Updates are not invoked after the final iteration.}
    \label{tab:adapter-frequency}
    \begin{tabular}{@{}lccc@{}}
        \toprule
        \textbf{Configuration} &
        \textbf{Final score} &
        \textbf{Adapter calls} &
        \textbf{Total tokens} \\
        \midrule
        No adaptation &
        $0.713 \pm 0.027$ & 0 & 853K \\
        Event-triggered adaptation &
        $0.953 \pm 0.008$ & 10 & 865K \\
        Periodic adaptation ($k=5$) &
        $0.939 \pm 0.010$ & 5 & 861K \\
        Every-iteration adaptation &
        $0.954 \pm 0.003$ & 29 & 877K \\
        \bottomrule
    \end{tabular}
\end{table}

\paragraph{Feedback access with and without strategy adaptation.}
Table~\ref{tab:feedback-ablation} shows that rich feedback improves mean scores both with and without the Adapter, by $0.056$ and $0.034$, respectively. Both configurations also exhibit lower standard deviations with diagnostic feedback. These observations suggest that information beyond a scalar score can benefit the optimization process even when strategy guidance is not updated; because generation and other downstream decisions also receive this information, the comparison does not isolate immediate Controller selection.

The Adapter remains beneficial under basic feedback: its mean-score advantage is $0.218$, compared with $0.240$ under rich feedback. Thus, diagnostic information is helpful but is not a prerequisite for the observed adaptation benefit. The difference between these two gains is $0.022$; with five runs and substantial variability in the basic-feedback condition, we treat this as descriptive evidence rather than an established interaction effect. Token consumption varies across the feedback and adaptation conditions, so the comparison concerns performance under a common iteration cap rather than equal information-processing cost. It does not establish an upper bound on adaptive performance determined by evaluator quality.

\begin{table}[ht]
    \centering
    \small
    \setlength{\tabcolsep}{4pt}
    \renewcommand{\arraystretch}{1.12}
    \caption{Feedback access on Heilbronn Triangle over five runs. Both conditions use the same evaluator and final scoring rule. Scores are mean $\pm$ standard deviation; token counts are mean total API consumption per run.}
    \label{tab:feedback-ablation}
    \begin{tabular}{@{}lcccc@{}}
        \toprule
        & \multicolumn{2}{c}{\textbf{Final score}}
        & \multicolumn{2}{c}{\textbf{Total tokens}} \\
        \cmidrule(lr){2-3}
        \cmidrule(lr){4-5}
        \textbf{Feedback} &
        \textbf{Full \ours{}} &
        \textbf{w/o Adapter} &
        \textbf{Full \ours{}} &
        \textbf{w/o Adapter} \\
        \midrule
        Basic &
        $0.897 \pm 0.057$ &
        $0.679 \pm 0.126$ &
        809K & 799K \\
        Rich &
        $0.953 \pm 0.008$ &
        $0.713 \pm 0.027$ &
        865K & 853K \\
        \bottomrule
    \end{tabular}
\end{table}

\paragraph{Observed sensitivity to the initial composition.}
Table~\ref{tab:initialization-ablation} compares fixed and adaptive optimization from three initial configurations. Adaptive optimization improves the mean score in every case, with absolute gains of $0.141$, $0.130$, and $0.094$ for local revision, diagnosis and repair, and exploration and recombination, respectively. Notably, adaptation also improves over the strongest fixed configuration, whose mean score is $0.857$. The benefit is therefore not confined to the default fixed reference used in the core ablation.

The range of mean scores across initializations decreases from $0.044$ for fixed compositions to $0.005$ for adaptive runs. This suggests that subsequent decisions can reduce the influence of the initial mechanism choice on terminal quality. Because only the first iteration is constrained in the adaptive variants, the result concerns sensitivity to initialization rather than recovery from a prolonged unsuitable strategy. Similar terminal means also do not imply identical intermediate decisions, resource use, or statistical equivalence across initializations.

\begin{table}[ht]
    \centering
    \small
    \setlength{\tabcolsep}{5pt}
    \renewcommand{\arraystretch}{1.12}
    \caption{Initialization sensitivity on Heilbronn Triangle. Each entry reports the mean and standard deviation over five runs with a maximum of 30 iterations. Adaptive \ours{} begins with the corresponding initial composition and can adjust it after the first iteration.}
    \label{tab:initialization-ablation}
    \begin{tabular}{@{}lcc@{}}
        \toprule
        \textbf{Initial configuration} &
        \textbf{Fixed} &
        \textbf{Adaptive \ours{}} \\
        \midrule
        Local-revision-first &
        $0.813 \pm 0.017$ &
        $\mathbf{0.954 \pm 0.008}$ \\
        Diagnosis-and-repair-first &
        $0.826 \pm 0.012$ &
        $\mathbf{0.956 \pm 0.008}$ \\
        Exploration-and-recombination-first &
        $0.857 \pm 0.011$ &
        $\mathbf{0.951 \pm 0.008}$ \\
        \bottomrule
    \end{tabular}
\end{table}

\section{Limitations and Broader Impact}
\label{app:limitations}
The profile comparison does not establish equivalence to all original optimizers, and the official-code checks cover only two methods on seven tasks. The motivation study is limited to one benchmark and three constructed situations per state. Complete per-run records are unavailable for the overall matrix, limiting uncertainty analysis. Common iteration caps do not control total computation, and the token-aware case study shows that substantial additional search cost can occur after the final improvement. The experiments do not establish long-horizon scaling, monotonic adaptation, causal effects of individual action changes, or satisfaction of the near-optimal selection assumption. Evaluation quality remains limited by each harness, including model-judge variability and benchmark-specific validity checks.

\end{document}